\documentclass[onecolumn]{elsarticle}

\usepackage{hyperref}
\usepackage{subfigure,caption}
\usepackage{xcolor}
\usepackage{bbm}
\usepackage{amsmath, amsthm, amssymb}
\DeclareMathOperator*{\argmax}{arg\,max}
\usepackage{mathtools}
\DeclarePairedDelimiter{\ceil}{\lceil}{\rceil}

\usepackage[normalem]{ulem}

\journal{Expert systems with applications}

\biboptions{authoryear}
\biboptions{authoryear, semicolon,round}

\begin{document}

\begin{frontmatter}

\title{Benchmarking unsupervised near-duplicate image detection\tnoteref{t1,t2}}
\tnotetext[t1]{\textcopyright 2019. This manuscript version is made available under the CC-BY-NC-ND 4.0 license http://creativecommons.org/licenses/by-nc-nd/4.0/}
\tnotetext[t2]{Full-text published in Expert Systems with Applications available at https://doi.org/10.1016/j.eswa.2019.05.002}

\author[label1]{Lia Morra\corref{cor1}}
\ead{lia.morra@polito.it}

\author[label1]{Fabrizio Lamberti}
\ead{fabrizio.lamberti@polito.it}

\address[label1]{Department of Control and Computer Engineering, Politecnico di Torino, Torino, Italy}
\cortext[cor1]{Corresponding author}

\begin{abstract}
Unsupervised near-duplicate detection has many practical applications ranging from social media analysis and web-scale retrieval, to digital image forensics. It entails running a threshold-limited query on a set of descriptors extracted from the images, with the goal of identifying all possible near-duplicates, while limiting the false positives due to visually similar images. Since the rate of false alarms grows with the dataset size, a very high specificity is thus required, up to $1 - 10^{-9}$ for realistic use cases; this important requirement, however, is often overlooked in literature. In recent years, descriptors based on deep convolutional neural networks have matched or surpassed traditional feature extraction methods in content-based image retrieval tasks. To the best of our knowledge, ours is the first attempt to establish the performance range of deep learning-based descriptors for unsupervised near-duplicate detection on a range of datasets, encompassing a broad spectrum of near-duplicate definitions. We leverage both established and new benchmarks, such as the Mir-Flick Near-Duplicate (MFND) dataset, in which a known ground truth is provided for all possible pairs over a general, large scale image collection. To compare the specificity of different descriptors, we reduce the problem of unsupervised detection to that of binary classification of near-duplicate vs. not-near-duplicate images. The latter can be conveniently characterized using Receiver Operating Curve (ROC). Our findings in general favor the choice of fine-tuning deep convolutional networks, as opposed to using off-the-shelf features, but differences at high specificity settings depend on the  dataset and are often small. The best performance was observed on the MFND benchmark, achieving 96\% sensitivity at a false positive rate of $1.43 \times 10^{-6}$.
\end{abstract}


\end{frontmatter}

\section{Introduction}
%
%
%
%

Near-duplicate (ND) image detection or discovery entails finding altered or alternative versions of the same image or scene in a large scale collection. This technique has plenty of practical applications, ranging from social media analysis and web-scale retrieval, to digital image forensics. Our work was motivated in particular by applications in the latter domain, as detecting the re--use of photographic material is a key component of several passive image forensics techniques. Examples  include detection of copyright infringements \citep{zhou2017effective, chiu2012video,ke2004efficient}, digital forgery attacks such as cut-and-paste, copy-move and splicing \citep{chennamma2009robust,hirano2006industry}, analysis of media devices seized during criminal investigations  \citep{connor2016quantifying, battiato2014aligning}, tracing the online origin of sequestered content \citep{de2016multiple, amerini2017tracing}, and fraud detection \citep{li2018anti,cicconet2018image}. 

In all the above-mentioned applications, we cannot resort to standard hashing techniques, given that even minimal alterations would make different copies untraceable. Similarly, it is not possible to rely on associated text, tags or taxonomies for retrieval, as done for instance in \citep{gonccalves2018semantic}, since they would likely change in different sites or devices where content is used. Images may be subject to digital forgery, with parts of one or more existing images combined to create fake ones. Therefore, it is imperative to resort to content-based image retrieval techniques for the task of locating near-duplicates. 

Let us consider, as a motivating example, the case of fraud detection. Many companies, like insurance ones, are relying on user-supplied photographic evidence to support business processes \citep{li2018anti}. Photos of the same object or scene may be re-used multiple times to obtain an unfair advantage: such frauds are unlikely to be detected unless a largely automatic image analysis system is in place.

It should be noticed that we are adopting a very broad definition of ND, encompassing all images of the same object or scene, whereas many papers in literature restrict the definition to copies of the same digital sources that have been digitally manipulated \citep{connor2015identification, foo2007pruning}. Naturally occurring NDs, such as images of the same scene or object acquired at different times or from different viewpoints, are often more challenging to detect. However, in emerging applications such as fraud detection, which motivate our work, we do not wish to restrict ourselves to either definition: as a matter of fact, we have no reason to assume that, when constructing fraudulent claims, digital content manipulation is more likely than simply acquiring different shots of the same scene. This broad definition brings ND detection closer to the task of instance-level image retrieval, which is abundantly studied in the literature, but with a crucial difference: while the latter is usually formulated as a \textit{human-guided} supervised search, the former needs as little human supervision as possible. To achieve this goal, we need to re-frame the problem from a \textit{supervised} $K$--nearest neighbors search to an \textit{unsupervised} threshold-limited search, where the distance is used as a \textit{classification} function to distinguish ND from non-ND pairs.

Realistic datasets in image forensics and fraud detection range between $10^{5}$--$10^{7}$ images \citep{connor2016quantifying}. Since the number of possible pairs grows quadratically with the dataset size, a very low false positive rate (or conversely, a high specificity) is needed to obtain a tractable number of false alarms and therefore be acceptable by the end user. For a dataset of $10^{6}$ images, a false positive rate of $10^{-9}$, which would be considered exceptionally low in many applications, would still translate to 500 false alarms. 

In recent years, deep convolutional neural networks (CNNs) have shown unprecedented performance in many computer vision tasks, and content-based image retrieval is no exception. To the best of our knowledge, very few papers have exploited CNNs-based descriptor for ND detection, but if we look at the closely related task of instance-level retrieval, a consistent body of research has emerged in recent years favoring the adoption of CNN-based representation over traditional SIFT-based methods \citep{zheng2017sift}. Experimental results on several benchmark datasets show that they achieve better performance, use more compact representations and are faster to compute \citep{zheng2017sift}. However, given the need to re-frame the problem as an \textit{unsupervised} threshold-limited search (where the overall performance is dominated by specificity rather than sensitivity), it is not straightforward to evaluate whether unsupervised near-duplicate search lies within the grasp of the current state-of-the-art. To the best of our knowledge, only \cite{connor2016quantifying} have previously addressed the issue of quantifying the performance of unsupervised near-duplicate detection \citep{connor2016quantifying}, and ours is the first contribution to specifically characterize deep learning descriptors on a wide range of ND categories.

One of the underlying reasons for is certainly the lack of suitably annotated benchmark datasets, as well as of an established methodology to measure a descriptor's performance. It is crucial that benchmarks for ND detection include a sufficiently large number of \textit{negative queries}, i.e., images for which the absence of NDs has been established, in order to assess both specificity and sensitivity. In some cases, we can resort to digital transformations to simulate NDs, but this is not applicable to all transformations.

Instance-level retrieval benchmarks, such as the Oxford5k, Ukbench and Holidays datasets, comprise a variety of naturally occurring and challenging NDs, but are rather small scale and include only clusters of related images \citep{zheng2017sift}. Recently, a new benchmark has become available to address the specific needs of ND detection: the Mir-Flickr Near Duplicate (MFND) dataset, based on the pre-existing MIR-Flickr collection \citep{connor2015identification}. In this benchmark,  {a large number of} NDs were mined using a semi-automatic procedure, so that the remaining images can be assumed to be negative queries;  {however, in their initial search Connor and colleagues focused on specific subclasses of NDs that limit the representativeness of this benchmark for applications such as fraud detection}.  \citet{connor2016quantifying} showed that the problem of unsupervised detection could thus be characterized as a binary classification problem, and we build upon their contribution for our experimental methodology.

The overarching objective of this study is to evaluate the performance of state-of-the-art deep learning descriptors and establish a baseline against which future research can be compared. A thorough experimental comparison includes a wide range of established and emerging public benchmarks, as well as data from a real-life fraud detection case study. Our contributions can be summarized as follows:
\begin{itemize}
\item we compare the performance of CNN-based descriptors on the task of unsupervised near-duplicate detection, and show empirically on a variety of datasets that specificity has a large impact on the relative ranking of different descriptors;
\item we extend considerably the available annotations for the MFND benchmark to obtain a large-scale benchmark which supports a wide range of ND definitions and use cases; 
\item finally, we extend previous work by \citet{connor2016quantifying} towards a principled evaluation methodology that captures the performance requirements of unsupervised ND discovery;
we show analytically and experimentally that by using hard negative mining, we can approximate the Area under the ROC curve ($AUC$) that can be used to rank the performance different descriptors.
\end{itemize}

The rest of the paper is organized as follows: in Section \ref{Related}, related work on instance retrieval and ND detection is reviewed. Section \ref{Datasets} introduces the datasets that are considered in the experiments. The evaluation methodology is presented in Section \ref{Perfo}, whereas the experimental setup is described in Section \ref{Experimental}. Results are presented and discussed in Sections \ref{Results} and \ref{Conclusion}, respectively.


\section{Related work}
\label{Related}

\subsection{Content-based image retrieval and instance-level retrieval}

Content-based image retrieval systems (CBIR) are designed to retrieve semantically similar images within a database based on a specific query (e.g., by providing another image). This problem can be decomposed in two main challenges: describing image content in terms of visual features, and conducting an exact or approximate nearest neighbor search based on a distance measure \citep{zheng2017sift,bay2008speeded}. Such features can be hand-crafted, or learned from data by using deep CNNs. In this section, we will review feature extraction techniques, and refer to existing surveys for the challenges related to feature aggregation, quantization, indexing and distance measures \citep{zheng2017sift,zhou2017recent}.  

\subsubsection{Hand-crafted features}

\textit{Global features} based on the characteristics of the entire image (color, shape, texture, histogram, etc..) were extensively used in early CBIR systems. In the early 2000s, \textit{local feature extraction} emerged as a more effective alternative, which generally involves two key steps: key interest point detection and local region description. In the first step, key salient features in the image are identified with high repeatability, using dense sampling or more commonly by detecting local extrema in the scale-space domain (e.g. Difference of Gaussians, Hessian matrix, etc.). One or multiple descriptors are then extracted from the local region centered at each interest point, usually designed to be invariant to rotation changes and robust to affine distortions, addition of noise, illumination changes, etc. The most popular local feature descriptors are SIFT and SURF \citep{zheng2017sift}. SIFT-based approaches generally yield very large feature sets, in the order of the thousands per image. The Bag of Visual Word (BoVW) is the most common approach for feature reduction and quantization in CBIR and instance retrieval. 
\subsubsection{Deep learning approaches}
Since 2015, deep learning has become the state of the art approach to CBIR \citep{wan2014deep,gordo2016deep,balntas2016learning,babenko2015aggregating,zagoruyko2015learning}. Deep CNNs have the distinct advantage of learning hierarchical, high-level abstractions close to the human cognition processes.   Similarly to SIFT, CNNs can be trained to extract features from local regions of interests (patches), after detecting key interest points, which are then quantized e.g., using the BoVW \citep{zagoruyko2015learning, balntas2016learning}.  Alternatively, it is possible to extract semantic-aware features from the activations of  top convolutional layers in an image: it can be shown that such feature vectors can be interpreted as an approximate many-to-many region matching, without the need to explicitly extract key points, and with the advantage of obtaining faster and more compact representations.
To this aim, two fundamental approaches are available. In the first approach, feature extraction is based on pre-trained models, like the VGG network trained for object recognition, alone or in combination with traditional visual features  \citep{babenko2015aggregating,wan2014deep}.  In the second one, a CNN can be trained to learn a ranking function in an end-to-end fashion, mapping the input space to a target latent space such that the Euclidean distance in latent space approximates visual similarity \citep{gordo2016deep, wang2014learning}. In order to optimize a ranking loss, a special architecture called  {a Siamese network} is used \citep{wang2014learning, gordo2016deep,gordo2017end}. Usually, descriptors are pre-trained on ImageNet to learn image semantics, and then fine-tuned on a second training set with relevance information \citep{wang2014learning,gordo2016deep}. 

\subsection{Near-duplicate image detection}
\label{sec:nnid_related}

Several works have focused on near-duplicate image detection as a distinct application from content-based image retrieval \citep{chennamma2009robust,foo2007pruning,chum2008near,li2015mining,xie2014fast,hu2009coherent,liu2015variable,xu2010near,kim2015near,battiato2014aligning,zhou2017effective,cicconet2018image,chen2017real,connor2016quantifying}. 

In order to frame our contribution with respect to previous literature, a more precise working definition of near-duplicate is needed. Given the range of potential applications, it comes as no surprise that the definition of near-duplicate image is indeed quite varied. Starting from the work by  {\citet{foo2007pruning}}, two main sources of near-duplicates have been identified in the literature, namely \textit{identical} and \textit{non-identical near duplicates} \citep{foo2007pruning,connor2015identification,jinda2013california, chen2017real}. Identical near-duplicates (INDs) are derived from the same digital source after applying some transformations, including cropping and rescaling, changes in image format, thumbnail resizing,  {insertion} of logos or watermarks, or cosmetic changes (black \& white conversion, image enhancement and so forth).  

Non-identical near-duplicates (NINDs), on the contrary, are defined as images that share the same content (i.e., they depict the same scene or object), but with different illuminations, subject movement, viewpoint changes, occlusions, etc. \citep{foo2007pruning,jinda2013california} Detecting NINDs is deemed more challenging than INDs, and their definition is more subjective \citep{jinda2013california}; for these reasons, many authors have mostly targeted INDs. 

Depending on the type of ND targeted and the level of transformation involved, most papers in literature have either focused on global features or on SIFT features combined with BoWV quantization. Global features have been mostly used for IND detection \citep{connor2016quantifying,chen2017real,li2015mining}. Local descriptors, such as SIFT features combined with BoVW quantization, allow detecting more aggressive alterations (including NINDs), sub-image retrieval or image forgery (e.g. copy-move attacks).  Local descriptors are prone to false positive matches, as they do not take into account spatial coherence; to reduce false alarms, some authors have proposed pruning techniques to improve specificity and scalability \citep{foo2007pruning,liu2015variable}, whereas other authors have focused on post-query verification \citep{zhou2017effective,hu2009coherent,xu2010near}.

From an evaluation point of view, many papers framed the problem of  {ND} detection as a supervised $K-$nearest neighbor search, and few papers have addressed the issue of quantifying the specificity of descriptors when performing unsupervised, threshold-limited near-duplicate discovery \citep{connor2016quantifying,chen2017real,kim2015near}. The most relevant prior work is that by Connor et al, who proposed a method to evaluate the specificity of ND detectors and choosing the optimal distance threshold, based on Receiver Operating Curve (ROC) analysis \citep{connor2016quantifying}. A more in-depth analysis of this methodology, and the extensions that we propose, is available in Section \ref{Perfo}. Other authors have used small test sets to establish the optimal threshold, that was subsequently applied to a larger dataset \citep{chen2017real}. For instance, Chen et al. used bloom filtering and range queries to detect duplicate images under scaling, watermarking and format change transformation \citep{chen2017real}; for evaluation purposes, they estimated the percentage of false and correct rejections, as well as precision and recall curves, on a smaller \sout{case} dataset.

\section{Datasets}
\label{Datasets}

The experimental results presented in this paper were based on four image collections, including  {a private dataset} (Section \ref{sec:claims}) and three publicly available benchmarks (Sections \ref{sec:MFND}, \ref{sec:Cali-ND} and \ref{sec:Holidays}). The CLAIMS dataset was collected for insurance purposes, and therefore constitutes a realistic case study for fraud detection applications. The MFND dataset \citep{connor2015identification}, based on the MIR-Flickr image retrieval benchmark \citep{huiskes2008mir}, contains a variety of both INDs and NINDs. 
Both the California-ND and Holidays datasets contain personal holiday photos and, while much smaller in size,  {include} several challenging NIND examples \citep{jegou2008hamming,jinda2013california}. Examples of ND pairs of different  {complexity} are given in Fig. \ref{fig:exampleND}, whereas a summary of the datasets characteristics is reported in Table \ref{tab:datasets}.

\begin{figure*}[!t]
\centering
\includegraphics[width=\linewidth]{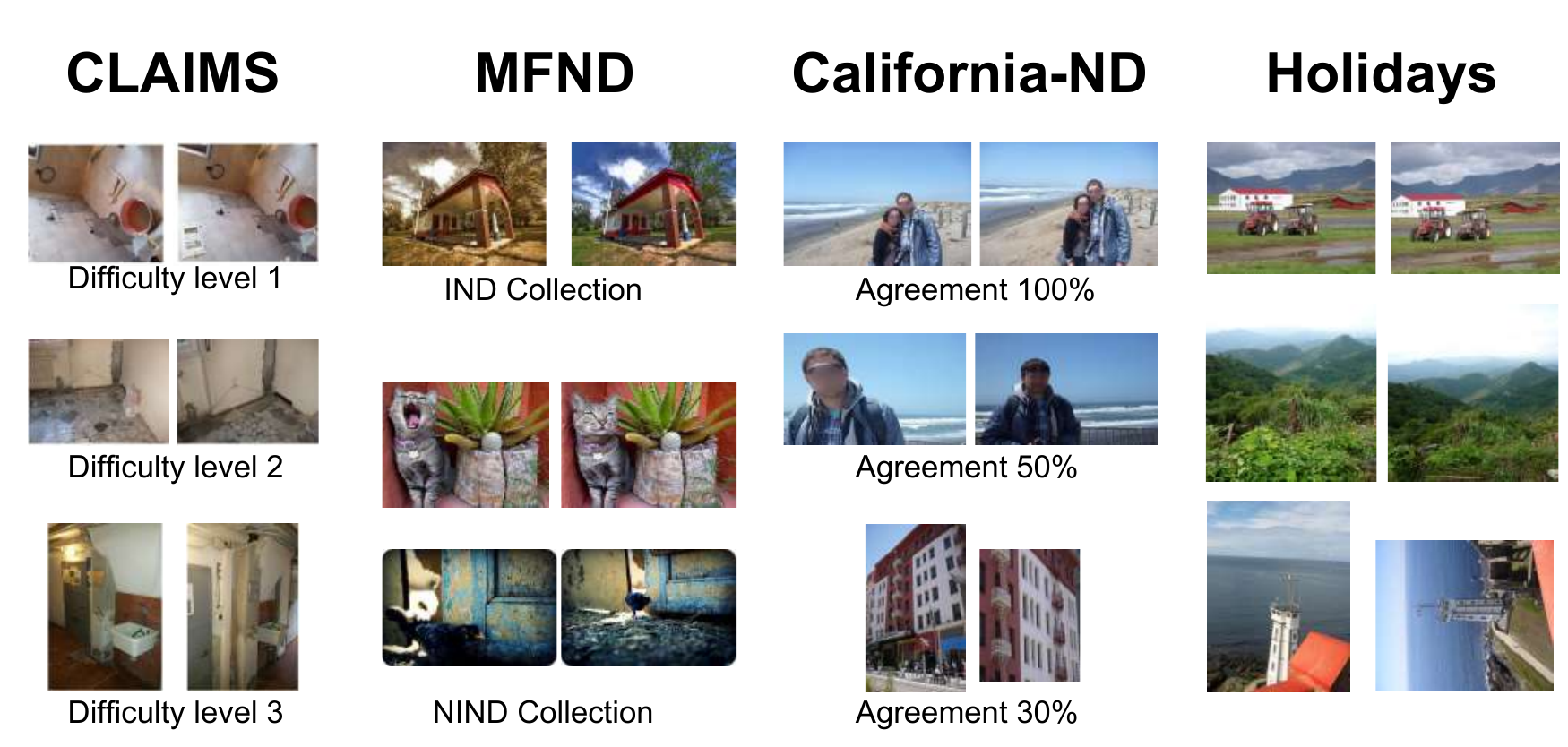}
\caption{Examples of near duplicates pairs of varying complexity from the four datasets included in the comparison. For the CLAIMS dataset, difficulty was evaluated subjectively by one rater, whereas for the California-ND, it was established based on the agreement between 10 independent raters.}
\label{fig:exampleND}
\end{figure*}

\begin{table}[t]
\centering
\resizebox{\textwidth}{!}{
\begin{tabular}{c | c | l | l}

\bfseries Dataset  &  \bfseries Size  & \bfseries IND clusters (pairs)  & \bfseries NIND clusters (pairs) \\
\hline
CLAIMS & 201,961 & NA &1037  (1,475) \\
\hline
MFND & 1,000,000 &   {3,825 (4,672)}   &  {10,454 (18,299)} \\
\hline
California-ND & 701 & NA &  107 (4,609) \\
\hline 
Holidays & 1,491 & NA & 500 (2,072) \\
\hline

\end{tabular}
}
\caption{Comparison of the benchmark datasets. Related IND or NIND pairs were  {grouped} into clusters; the average number of images per cluster ranges from 2.07 to 6.55. IND and NIND pairs are counted separately, where applicable.}

\label{tab:datasets}
\end{table}

\subsection{CLAIMS dataset}
\label{sec:claims}

The CLAIMS dataset includes a variety of indoor and outdoor scenes, mostly from residential and commercial  buildings. It contains a total of 201,961 images coming from 22,327 claims. Image subsets that are associated  {with} a given claim generally include images of the same scenes or objects, representing a source of relatively rapidly identifiable NINDs. More infrequently, images from different claims may also represent the same scene or object. This dataset contains both IND (e.g. insertion of small captions or logos, changes in aspect ratio, format change, compression, etc.) and NIND clusters (e.g. sequential snapshots of the same scene, viewpoint changes, etc.).

The collection was annotated to generate both positive queries (i.e., with  {known} NDs) and negative queries (i.e., for which absence of NDs was confirmed). NDs were annotated following a semi-manual procedure, in which a set of claims was randomly selected. For each claim, all potential image pairs were generated and the ND pairs were manually selected. Connected pairs of NDs from the same claim were grouped to form clusters.

Non-near duplicate (NND) pairs were randomly extracted following a hard negative mining strategy (see  {Section} \ref{secROCanalysis} for details). The results were visually inspected obtaining additional 103 near-duplicate pairs.  The final annotated set included 1,475 ND pairs, forming  1,037 distinct clusters; the average number of images per cluster is 2.2.

\subsection{MirFlickr Near Duplicate}
\label{sec:MFND}

The MIR-Flickr Near Duplicate (MFND) collection is a recent revisitation of the MIR-Flickr image retrieval benchmark \citep{huiskes2008mir}. Connor and colleagues observed a significant number of NDs in this one million image collection, which were semi-automatically retrieved using different ND finders \citep{connor2015identification}. We have expanded their annotations by adopting a broader definition of ND, as well as using different descriptors. 

The first MFND annotation was generated using a set of four global descriptors (based on MPEG-7 and perceptual Hashing global features) and five distance measures, which were  combined to form different similarity functions \citep{connor2015identification}. A threshold-limited nearest-neighbor search was conducted using approximated metric search techniques, yielding a few  {thousand potential ND} pairs for every function.  {We have expanded this annotation by using the three CNN-based descriptors included in this study, and the Euclidean distance. Several threshold-limited, $K$-nearest neighbor searches were performed (with $K$=5 and $K$=1), yielding a few hundred thousands potential ND pairs which were visually inspected.} Exact duplicates were eliminated based on the MD5 hash.

Each of the resulting image pairs  {was} manually assigned to one of three categories, IND, NIND or other, following the categorization illustrated in  {Section \ref{sec:nnid_related}}  \citep{connor2015identification}. 
The strength of this methodology is that it minimizes biases with respect to the images in the collection,  {as well as to} the method with which the near-duplicates have been  {detected}.  We assumed, as in previous work \citep{connor2015identification}, that both IND and NIND relations are transitive, allowing the identification of clusters of images that share the same content.  {The resulting clusters were also visually inspected for consistency.}

As for the CLAIMS collection, NND pairs were generated through a hard negative mining procedure; results were visually inspected identifying 120 additional NIND pairs. 

The available annotations were thus substantially extended from 1,958 to 3,825 IND clusters (4,672 vs. 2,407 pairs) and from 379 to 10,454 NIND clusters (18,299 pairs). Many new IND pairs detected were subject to digital content manipulations, cropping or color alterations; we found that CNN-based descriptors were particularly robust to colorization techniques. A total of 30,925 images were found to have at least one IND or NIND in the collection, with a mean cluster size of 2.2.

\subsection{California-ND}
\label{sec:Cali-ND}

The California-ND collection comprises 701 photos taken from a real user's personal photo collection \citep{jinda2013california}. It includes many challenging NIND cases, without resorting to artificial image transformations. To account for the intrinsic ambiguity of NIND definition, the collection was manually annotated by 10 different observers, including the photographer himself. Instructions such as ``If any two (or more) images look similar in visual appearance, or convey similar concepts to you, label them as near-duplicates.'' were given to the raters.
Out of 245,350 unique possible combinations, 4,609 image pairs were identified as ND by at least one subject; notably, in 82\% of the cases raters disagreed to some extent on whether or not a pair of images should be considered ND. The image pairs form 107 clusters of NIND images, where each cluster contains on average 6.55 images; the ND pairs were grouped assuming that the ND relationship is transitive (which is not generally the case, but seemed reasonable in this particular situation). 

\subsection{Holidays}
\label{sec:Holidays}

The INRIA Holidays dataset \citep{jegou2008hamming}, a popular benchmark for instance retrieval, is mainly composed by the authors' personal holidays photos. The images, all high resolution, include a large variety of scene types (natural, man-made, water and fire effects, etc). Images were grouped by the authors in 500 disjoint image clusters, each representing a distinct scene or object, for a total of 1,491 images and an average cluster size of 2.98 images. From the 500 clusters, a total of 2,072 ND pairs, mostly NIND, can be identified.

\section{Performance evaluation}
\label{Perfo}
In this section, we illustrate the performance metrics and protocol used to evaluate the specificity and sensitivity of unsupervised ND detection. In  {the} first battery of test, we extended the work of Connor and colleagues \citep{connor2016quantifying}, reducing the problem of unsupervised discovery to that of binary classification; ROC analysis can be used to measure the ability of a ND detector to distinguish ND pairs from visually similar examples, as detailed in Section \ref{secROCanalysis}. The second battery of test involves estimating the average false positive rates generated by a negative query, and is explained in Section \ref{secFROCanalysis}; the relationship between these two performance measures is also explored.  An overview of the methodology is presented in Fig. \ref{fig:methods}.

\begin{figure*}[th!]
\centering
\includegraphics[width=\linewidth]{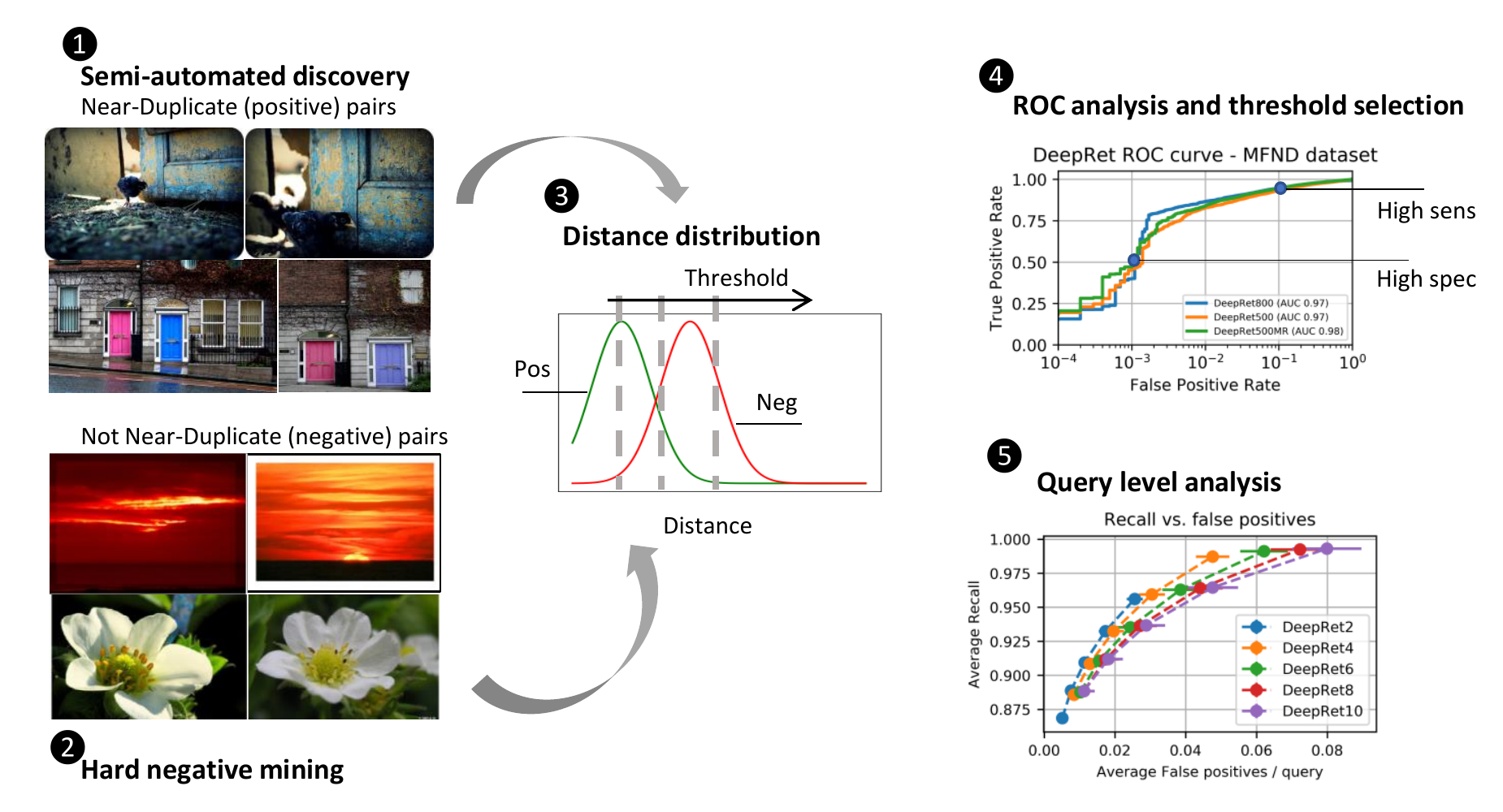}
\caption{Methodology employed to calculate the performance on the MFND benchmark. First, all near-duplicate pairs are discovered through a semi-supervised search technique (step 1). On the  {remainder} of the collection, hard negative mining is used to identify hard samples of visually similar, but not  {near-duplicate pairs} (step 2); crucially, this step needs to be repeated  {for each descriptor}. Using the distance as a classification function (step 3), ROC analysis can be used to characterize the ability of the descriptor to distinguish near-duplicates from not-near duplicate pairs. From ROC analysis, suitable thresholds on the distance can be selected based on the application requirements. Performance at query time can be thus be reliably estimated (step 5). }
\label{fig:methods}
\end{figure*}

\subsection{ROC analysis}
\label{secROCanalysis}

As suggested by  \citet{connor2016quantifying}, a near-duplicate finder can be modeled as a positive numeric function  $D$ over any two image descriptors $x$ and $y$, where normally $D$ will be a proper distance metric. To run an unsupervised search, it is necessary to use $D$ as a classification function over images pairs, which without loss of generality can be achieved by choosing a threshold $t$:
\begin{equation}
D_{t} \left ( x , y \right ) = D \left ( x , y \right ) < t
\end{equation} 
 
The problem of unsupervised discovery can be characterized as finding the near-duplicate intersection of two image sets $X \cap_{ND} Y$, that is the set of pairs of images from sets $X$ and $Y$ that satisfy the conceptual near-duplicate relation ND \citep{connor2016quantifying}.   If $Sens\left( t \right )$ and  $Spec\left( t \right )$ are the sensitivity and specificity of $D_{t} \left ( x , y \right )$,  the number of true positive (TP) matches will be
\begin{equation}
TP\left ( t  \right ) =Sens \left ( t  \right ) \left | X \cap_{ND} Y \right | 
\end{equation}
and the number of false positives (FP) will be
\begin{equation}
FP\left ( t  \right ) =\left ( 1 - Spec \left ( t  \right ) \right )\left | X  \right | \left | Y  \right | 
 \label{eqFP}
\end{equation}
assuming that $ \left | X \cap_{ND} Y \right |  <<  \left | Y \right |$. In our setting, $\left | X \right | = K$ is the number of query images and $\left | Y \right | = M$ is the size of the collection. Another useful figure to define is the number of false positives / query image, which can be computed as
\begin{equation}
FP_{q}\left ( t  \right ) =\left ( 1 - Spec \left ( t  \right ) \right ) M
 \label{eqFPq}
\end{equation} 

Given a set of ND and NND pairs, the sensitivity (or recall) and the specificity of a ND detector can be estimated as follows:
\begin{equation}
Sens\left ( t  \right ) = \frac{\textnormal{No. of correctly identified ND pairs}}{ \textnormal{Total no. of ND pairs}}  
\end{equation} 
\begin{equation}
Spec\left ( t  \right ) = \frac{\textnormal{No. of correctly identified NND pairs }}{ \textnormal{Total no. of NND pairs}}  
\end{equation} 

Both quantities are function of the threshold $t$, and the overall performance can be characterized by ROC analysis.

\subsection{Hard negative mining}
\label{hardnegsec}
In a realistic dataset the pool of NND images is very large, compared to the number of ND pairs - it is not  {feasible} to evaluate all possible pairs. Hard negative mining extracts a compact set of NND from a large image collection, starting from a subset of randomly selected query images, for which we can  assume that a near-duplicate match does not exist in the collection. For each query image, the pairwise distances between the query images and all the other images in the collection are calculated, and the most "difficult" examples are selected. 

Starting from a random sample of query images, two hard negative mining strategies were considered:
\begin{itemize}
\item the nearest neighbor for each query is selected (\textit{hn1});
\item the $K-$nearest neighbors for all  {queries} are retrieved and sorted; the most difficult pairs (i.e., those with the smallest distances) are then selected (\textit{hn2}).
\end{itemize}

Notably, the hard negative mining procedure depends on the relative ranking of the images, and hence has to be repeated for each descriptor and for each distance formulation.

The distances of the hard negatives are among the smallest of the $K \times M$ distances measured, where $K$ is the number of query images and $M$ is the dimension of the dataset: if a distance threshold exists such that all the ``difficult'' pairs are successfully identified,  {then} we can assume that all potential NND pairs in the collection will be identified as well. For instance, for the CLAIMS dataset $K = 4400$ and $M = 80,000$, yielding a specificity of $ 1 - 1 / 353,000,000 = 1 - 2.83 \times 10 ^{-9} $,  {which is the smallest possible sensitivity that can be measured in this setting}.

The specificity measured on the hard negative samples can be used to approximate the true specificity that would be observed at large. For instance, a .9 specificity (.1 FP rate), would allow to successfully discard $0.9 M$ NND pairs; however,  {it} would also fail to discard at least $0.1 K$ NND pairs, and hence would correspond to a specificity of at most $1 - 1.25 \times 10  ^{-6} $. In this way, it is possible to estimate a lower bound on the amount of FPs generated on datasets of arbitrary size.

\subsubsection{Area under the ROC curve}
The AUC is a common summary metrics that quantifies the global performance of a classifier \citep{bradley1997use}. We estimated the AUC  {for each descriptor and }95\% confidence intervals were calculated under the normal assumption according to  { \citet{hanley1982meaning}}.

Since the ROC is calculated on a subset of all possible negative pairs, the resulting AUC will be an approximation of the true AUC if all pairs were taken into account. We will refer in the following to $AUC_{hn1}$ and  $AUC_{hn2}$ to denote the $AUC$ calculated on pairs extracted using hard negative mining strategies $hn1$ and $hn2$, respectively.

\newcommand{\1}[1]{\mathbbm{1}_{#1}}

Let $p_1,...,p_{N^+}$ be ND pairs (i.e., positive samples) and $n_1,...,n_{N^-}$ be all the NND pairs (i.e., negative samples), where in our case $N^- = K \times M$. The AUC can be expressed as a sum of indicator functions \citep{hanley1982meaning}:
\begin{equation}
    AUC = \frac{1}{N^+N^-}\sum_{i=1}^{N^+} \sum_{j=1}^{N^-}\1{f(p_i) > f(n_j)}
    \label{eq:auc}
\end{equation}
where $f(\cdot)$ is a scoring function which, in our case, is the distance between the descriptors of the two images in the pair \footnote{We follow here the notation normally used in ROC literature where the positive samples are expected to be scored higher than negative samples, whereas in our case the scoring function is a distance and pairs with lower distance would be scored higher}. For simplicity, we omit $f(\cdot)$ from the notation in the rest of the paper. 

Let $n_l,...,n_{H^-}$ be the hard negative NND pairs, where $H^- \ll N^-$. The estimated AUC can be expressed as follows:
\begin{equation}
    AUC_{hn} = \frac{1}{N^+H^-}\sum_{i=1}^{N^+} \sum_{l=1}^{H^-}\1{p_i > n_l}
    \label{eq:auch}
\end{equation}

Since in general
\begin{equation}
    \1{p_i > n_j} \leq \1{p_i > n_l}  \quad  \forall n_j \in \mathbbm{N^- - H^-} \quad  \forall n_l \in \mathbbm{H^-}
    \label{eq:inequality}
\end{equation}
it can be demonstrated that for both hard negative mining strategies $AUC_{hn}$ is an upper bound for the true $AUC$. It follows that the most appropriate choice would be to use the strategy that provides the tighter bound. Analytical proof is provided in \ref{sec:app1}.

\subsection{Range query performance}
\label{secFROCanalysis}

ROC analysis does not directly represent the observed system performance, which also depends on the distribution of the type of images, the size of the clusters, and so forth. We analyzed an alternative performance measure, obtained by simulating the case of a single query image $x$ compared against a collection of images $Y$, which is a special case of the general problem of near-duplicate detection described in Section \ref{secROCanalysis}. An unsupervised, threshold-limited range search is conducted to retrieve a list of potential near-duplicates, and used to estimate the number of FPs/query or $FP_{q}\left ( t  \right )$. In practice, it is convenient to restrict the search to the $K$--nearest neighbors  {in order} to cap the number of FPs/query to a reasonable number. 
The proposed experimental setup executes a number of positive queries (i.e., images with one or more known ND), and negative queries (i.e., images that have no expected NDs), over a dataset constructed as follows: 
\begin{itemize}
\item positive queries were derived from the clusters of ND images, where the first image are used as queries and the rest are inserted in the database, as normally done for Holidays and other image retrieval benchmarks;
\item negative query images were selected from the NND pairs, and a set of distractors are used to evaluate specificity; in practice,  {we use for convenience} the same image pool used for hard negative mining. 
\end{itemize}

For varying values of the threshold $t = T_{i}$ on the distance measure, we compared the \textit{average recall}, calculated over all positive queries, and the \textit{average number of FPs/query}, calculated over all negative queries. Note that average recall is different from pair-wise sensitivity used in ROC analysis, as each query may contain multiple pairs of varying ``difficulty''. 
The FPs/query depend on the size of the dataset and the specificity as predicted by Equation \ref{eqFP}.

\section{Experimental setup}
\label{Experimental}
In this section, a detailed analysis of the experimental  {setup} is given concerning the descriptors selection, their implementation {,} and the hard negative mining parameters. 

\subsection{Descriptors}
\label{Descriptors}

Two sets of descriptors were compared in this work: global descriptors, and CNN-based descriptors; for the latter, we compared examples of the two main approaches (aggregation of raw deep convolutional features without embedding and Siamese networks) described in Section \ref{Related}. 

Among global descriptors, \textit{GIST} \citep{oliva2001modeling} was selected based on previous results on the MFND collection \citep{connor2016quantifying}. The GIST, or spatial envelope,  {is} a bio-inspired feature that simulates human visual perception to extract rough but concise context information \citep{oliva2001modeling}. The input image  {is}  decomposed using spatial pyramid into $N$ blocks, filtered by a number of multi-scale, multi-orientation Gabor filters (4 scales, 8 orientations per scale), and then summarized by a feature extractor that captures the ''gist'' of the image, handling translational, angular, scale and illumination changes. We experimented with perceptual Hashing, however the results are not reported as they were generally very poor. 

The \textit{ SPoC (Sum-Pooled Convolutional)} descriptor was initially proposed by  {\citet{babenko2015aggregating}}. The features are extracted from the top convolutional layer of a  {pre-trained} neural network and spatially aggregated using sum pooling. The length of the feature vector will thus be equal to the depth of the final convolutional layer (usually in the order of the hundreds). Best results were obtained extracting features after ReLU activation, confirming previous findings \citep{babenko2015aggregating}. PCA whitening and compression is applied, and the vectors are normalized to unit length (L2 normalization).

The \textit{R-MAC} architecture, proposed by \citet{tolias2015particular} builds a compact feature vectors by encoding several image regions in a single pass. First, sub-regions are defined using a fixed grid over a range of progressively finer scales $l$ ranging from 1 to $L$; then, max-pooling is used to extract features from each individual region. Each region feature vector is post-processed with PCA-whitening and L2 normalization. Finally, the regional feature vectors are summed into a single image vector, which is again L2 normalized.

The \textit{DeepRetrieval} architecture, proposed by  {\citet{gordo2016deep}}, employs a Siamese network to learn a ranking function based on the triplet loss function. Let $I_{q}$ be a query image with descriptor $q$, $ I_{+}$ be a relevant image with descriptor  $ d_{+}$, and  $ I_{-}$ be a non-relevant image with descriptor  $ d_{-}$. The ranking triplet loss is defined as
\begin{equation}
L\left ( I_{q}, I_{+},  I_{-} \right ) = \frac{1}{2}max(0, m + \left \|  q - d_{+} \right \|^{2} - \left \|  q - d_{-}\right \|^{2} )
\end{equation}
where $m$ is a scalar that controls the margin. At test time, the features are extracted from the top convolutional layer and aggregated using sum-pooling and normalization. The Deep Retrieval architecture includes an additional proposal network, similar to the R-MAC grid network, so that the features are calculated on several potential regions of interest, as opposed to the entire image. The Deep Retrieval network is pre-trained on the Landmarks dataset \citep{babenko2014neural}.

\subsection{Implementation}
\label{implementation}

 The tested descriptors, and related parameters, are summarized in Table \ref{tab:desc}. For SPoC and GIST, images were resized to $512 \times 512$, whereas for Deep Retrieval images were rescaled so that the longest side is equal to S.
 
 A Python re-implementation of the original Matlab code by Olive and Torralba was used for GIST, after converting images to grayscale\footnote{http://people.csail.mit.edu/torralba/code/spatialenvelope/}. 
The SPoC descriptor was computed from pre-trained networks architectures such as VGG \citep{simonyan2014very} and Residual Networks \citep{he2016deep}. We included both models pre-trained on ImageNet \citep{deng2009imagenet}, as well as on the Places205 or Places365 datasets \citep{zhou2014learning,zhou2017places}, and on a hybrid dataset including images from both ImageNet and Places.  {The R-MAC descriptor was re-implemented in Python based on the original Matlab implementation by the authors\footnote{https://github.com/gtolias/rmac}; R-MAC was calculated only for the ResNet101 architecture.}
All networks were available in Caffe; pre-trained models were downloaded from the Caffe Model Zoo\footnote{https://github.com/BVLC/caffe/wiki/Model-Zoo}, or were made available by the authors for the Places dataset\footnote{https://github.com/CSAILVision/places365}$^{,}$\footnote{http://www.europe.naverlabs.com/Research/Computer-Vision/Learning-Visual-Representations/Deep-Image-Retrieval}.  

We trained the PCA parameters, without reducing the number of features, on a representative set of the collection (95,230 for CLAIMS and 100,000 for MFND), which was not used in testing. The parameters trained on the MFND collection were also used for the Holidays and California-ND collection; although the image characteristics in the two datasets is different, this is consistent with previous works in literature \citep{jegou2008hamming}.

The FAISS library \citep{johnson2017billion}, specifically the flat index with L2 exact search, was used to index the collection and perform range queries. All experiments were run on a system with a i7-7700 CPU @3.60GHz and GTX1080Ti nVIDIA GPU.

\begin{table*}[th!]

\centering
\resizebox{\textwidth}{!}{
\begin{tabular}{c|c|c|c}
\bfseries Descriptor & \bfseries Label   & \bfseries Size  & \bfseries  Parameters
 \\
\hline\hline
\hline
 GIST \citep{oliva2001modeling} & GIST4 & 512 &  number of blocks = 4 \\
GIST \citep{oliva2001modeling} & GIST8 & 512 &  number of blocks = 8 \\
\hline
 Deep Retrieval \citep{gordo2016deep}  & DeepRet800 &2048 &  {ResNet101, Fine-tuned on Landmarks dataset}, S = 800, no multiresolution \\
Deep Retrieval \citep{gordo2016deep} & DeepRet500 & 2048 &  {ResNet101, Fine-tuned on Landmarks dataset}, S = 500, no multiresolution \\
Deep Retrieval \citep{gordo2016deep} & DeepRet500MR &  2048 &  {ResNet101, Fine-tuned on Landmarks dataset}, S = 500, multiresolution (2) \\
\hline
SPOC \citep{babenko2015aggregating}  & SP\_VGG19IN &512 & VGG19, Trained on ImageNet \\
SPOC \citep{babenko2015aggregating}  & SP\_VGG16PL &512 & VGG16, Trained on Places205 \\
aSPOC \citep{babenko2015aggregating}  & SP\_VGG16HY &512 & VGG16, Trained on Hybrid (Places205 \& ImageNet) dataset \\
SPOC \citep{babenko2015aggregating}  & SP\_ResNet101IM &2048 & ResNet101, Trained on ImageNet dataset \\
SPOC \citep{babenko2015aggregating}  & SP\_ResNet152IM &2048 & ResNet152, Trained on ImageNet dataset \\
SPOC \citep{babenko2015aggregating}  & SP\_ResNet152HY &2048 & ResNet152, Trained on Hybrid  (Places365 \& ImageNet) dataset \\
\hline
 {R-MAC \cite{tolias2015particular}} & RMAC & 2048 & ResNet101, Trained on ImageNet dataset, $L$=2  \\

\hline
\hline
\end{tabular}}
\caption{Synthetic description of the descriptors used.}
\label{tab:desc}
\end{table*}

\subsection{Hard negative mining}
\label{hnimpl}
For each dataset (CLAIMS and MFND), we randomly selected a set of negative query images (4,500 and 5,000, respectively), and a larger pool for mining (80,000 and 70,000 images, respectively), after excluding IND or NIND pairs and images used to train the PCA parameters. The two hard negative mining strategies introduced in Section \ref{hardnegsec}, were compared: in \textit{hn2}, the 10 nearest neighbors were found for each query images, and then the 10,000 most difficult pairs were selected. The number of samples was increased for \textit{hn2} to account for images that belong to more than one pair, which we observed experimentally, and ensure sufficient diversity. For GIST and Deep Retrieval, the hard negative mining procedure was calculated for one parameter set, to reduce the computational cost. For the Holidays and California-ND datasets, we used the NND pairs mined for MFND; this  is consistent with previous works that have used the same collection as distractors for large scale retrieval testing \citep{jegou2008hamming}.
For both datasets, the hard negative mining procedure was repeated for a large number of descriptors (8 for MFND, and 13 for the CLAIMS dataset), and the results were visually inspected for the presence of near duplicates. A total of 101 (2.2\%) and 121 (2.4\%) ND pairs were found for the CLAIMS and MFND datasets, respectively, and their labels were changed accordingly. 

\section{Results}
\label{Results}
In this section, results of the ROC analysis are compared across different descriptors (Section \ref{resROC}) and different datasets (Section \ref{datasetcomp}). Finally, the performance obtained from random queries is analyzed and compared with that predicted by ROC analysis in Section \ref{perfquery}.

\begin{table*}[th]

\centering
\resizebox{\textwidth}{!}{
\begin{tabular}{|c|c|c|c|c|c|}

\hline
\bfseries Descriptor & CLAIMS  &  \multicolumn{2}{c}{MFND} & California-ND & Holidays  \\
\hline
& \bfseries AUROC  & \bfseries AUROC-IND & \bfseries AUROC-all  & \bfseries AUROC & \bfseries AUROC    \\
\hline
GIST4 &	0.397	(	0.381	--	0.413)	&	0.808	(	0.799	--	0.817)	&	0.5	(	0.491	--	0.509)	&	0.598	(	0.587	--	0.610)	&	0.365	(	0.351	--	0.378)	\\
GIST8&	0.45	(	0.433	--	0.467)	&	0.854	(	0.847	--	0.862)	&	0.561	(	0.552	--	0.569)	&	0.696	(	0.685	--	0.706)	&	0.493	(	0.478	--	0.508)	\\
\hline
DeepRet800&	0.891	(	0.88	--	0.903)	&	0.994	(	0.993	--	0.996)	&	0.983	(	0.981	--	0.984)	&	0.929	(	0.924	--	0.934)	&	0.744	(	0.731	--	0.758)	\\
DeepRet500&	0.88	(	0.868	--	0.892)	&	0.992	(	0.991	--	0.994)	&	0.979	(	0.978	--	0.981)	&	0.917	(	0.911	--	0.923)	&	0.748	(	0.734	--	0.761)	\\
DeepRet500MR&	0.891	(	0.88	--	0.903)	&	0.992	(	0.99	--	0.994)	&	0.983	(	0.982	--	0.984)	&	0.938	(	0.933	--	0.943)	&	0.789	(	0.777	--	0.802)	\\
\hline
SP\_VGG19\_IM	&0.391	(	0.375	--	0.407)	&	0.934	(	0.929	--	0.940)	&	0.904	(	0.901	--	0.908)	&	0.88	(	0.873	--	0.887)	&	0.671	(	0.656	--	0.685)	\\
SP\_VGG16\_PL	&0.446	(	0.429	--	0.463)	&	0.907	(	0.901	--	0.914)	&	0.881	(	0.877	--	0.886)	&	0.909	(	0.903	--	0.915)	&	0.675	(	0.66	--	0.689)	\\
SP\_VGG16\_HY&	0.418	(	0.401	--	0.434)	&	0.929	(	0.923	--	0.934)	&	0.906	(	0.903	--	0.910)	&	0.896	(	0.89	--	0.903)	&	0.697	(	0.683	--	0.711)	\\
SP\_ResNet101IM&	0.518	(	0.501	--	0.536)	&	0.961	(	0.957	--	0.965)	&	0.941	(	0.938	--	0.943)	&	0.93	(	0.925	--	0.936)	&	0.776	(	0.763	--	0.789)	\\
SP\_ResNet512IM&	0.522	(	0.505	--	0.540)	&	0.963	(	0.959	--	0.967)	&	0.944	(	0.941	--	0.946)	&	0.921	(	0.916	--	0.927)	&	0.78	(	0.767	--	0.793)	\\
SP\_ResNet512HY&	0.459	(	0.442	--	0.476)	&	0.924	(	0.918	--	0.930)	&	0.916	(	0.913	--	0.920)	&	0.866	(	0.859	--	0.874)	&	0.737	(	0.723	--	0.751)	\\
\hline
RMAC&	0.323	(	0.308	--	0.338)	&	0.99	(	0.988	--	0.992)	&	0.945	(	0.942	--	0.947)	&	0.88	(	0.873	--	0.888)	&	0.737	(	0.723	--	0.751)	\\
\hline
Average & 0.549	&	0.937	&	0.870	&	0.863	&		0.684 \\
\hline
\end{tabular}}
\caption{Area under the ROC curve (AUC) with 95\% confidence intervals. The NND pairs are extracted using the first hard negative mining strategy (\textit{hn1}). For the MFND dataset, the AUC is calculated separately for both IND and NIND pairs (AUROC-all), and for IND pairs vs. NND pairs; in the latter case, NIND are not counted as either FP or TP. The average AUC across all descriptors provides a semi-quantitative estimate of the "difficulty" of each dataset. }

\label{tab:resROC}
\end{table*}

\begin{table*}[th]

\centering
\resizebox{\textwidth}{!}{
\begin{tabular}{|c|c|c|c|c|c|}

\hline
\bfseries Descriptor & CLAIMS  &  \multicolumn{2}{c}{MFND} & California-ND & Holidays  \\
\hline
& \bfseries AUROC  & \bfseries AUROC-IND & \bfseries AUROC-all  & \bfseries AUROC & \bfseries AUROC    \\
\hline
GIST4&	0.108	(	0.101	--	0.114)	&	0.706	(	0.696	--	0.715)	&	0.317	(	0.31	--	0.323)	&	0.398	(	0.389	--	0.408)	&	0.178	(	0.17	--	0.186)	\\
GIST8&	0.121	(	0.114	--	0.128)	&	0.756	(	0.747	--	0.765)	&	0.346	(	0.339	--	0.352)	&	0.462	(	0.452	--	0.472)	&	0.243	(	0.234	--	0.253)	\\
\hline
DeepRet800&	0.381	(	0.366	--	0.395)	&	0.99	(	0.988	--	0.992)	&	0.969	(	0.967	--	0.970)	&	0.929	(	0.924	--	0.934)	&	0.628	(	0.614	--	0.642)	\\
DeepRet500&	0.428	(	0.412	--	0.443)	&	0.987	(	0.985	--	0.989)	&	0.962	(	0.96	--	0.964)	&	0.917	(	0.911	--	0.923)	&	0.614	(	0.6	--	0.628)	\\
DeepRet500MR&	0.46	(	0.444	--	0.476)	&	0.987	(	0.985	--	0.989)	&	0.97	(	0.968	--	0.971)	&	0.938	(	0.933	--	0.943)	&	0.676	(	0.663	--	0.690)	\\
\hline
SP\_VGG19\_IM	&0.24	(	0.229	--	0.251)	&	0.882	(	0.876	--	0.889)	&	0.82	(	0.816	--	0.825)	&	0.787	(	0.779	--	0.796)	&	0.52	(	0.507	--	0.534)	\\
SP\_VGG16\_PL&	0.288	(	0.274	--	0.302)	&	0.866	(	0.859	--	0.873)	&	0.821	(	0.817	--	0.826)	&	0.859	(	0.851	--	0.866)	&	0.577	(	0.563	--	0.591)	\\
SP\_VGG16\_HY&	0.267	(	0.255	--	0.279)	&	0.881	(	0.874	--	0.887)	&	0.834	(	0.83	--	0.838)	&	0.814	(	0.806	--	0.822)	&	0.563	(	0.55	--	0.577)	\\
SP\_ResNet101IM	&0.397	(	0.382	--	0.412)	&	0.943	(	0.939	--	0.948)	&	0.911	(	0.908	--	0.914)	&	0.892	(	0.885	--	0.898)	&	0.685	(	0.671	--	0.698)	\\
SP\_ResNet512IM &	0.396	(	0.381	--	0.411)	&	0.947	(	0.943	--	0.952)	&	0.917	(	0.914	--	0.920)	&	0.885	(	0.878	--	0.891)	&	0.694	(	0.68	--	0.707)	\\
SP\_ResNet512HY	&0.327	(	0.313	--	0.340)	&	0.881	(	0.874	--	0.887)	&	0.866	(	0.862	--	0.870)	&	0.784	(	0.776	--	0.793)	&	0.627	(	0.613	--	0.641)	\\
\hline
RMAC	&0.336	(	0.322	--	0.350)	&	0.985	(	0.983	--	0.988)	&	0.917	(	0.914	--	0.920)	&	0.825	(	0.817	--	0.833)	&	0.641	(	0.627	--	0.655)	\\
\hline
Average & 0.312	&		0.901	&	0.804	&		0.791	&					0.554	\\		
\hline
\end{tabular}}
\caption{Area under the ROC curve (AUC) with 95\% confidence intervals. The NND pairs are extracted using the  {second} hard negative mining strategy (\textit{hn2}). For the MFND dataset, the AUC is calculated separately for both IND and NIND pairs (AUROC-all), and for IND pairs vs. NND pairs; in the latter case, NIND are not counted as either FP or TP. The average AUC across all descriptors provides a semi-quantitative estimate of the "difficulty" of each dataset.}
\label{tab:resROChn2}
\end{table*}

\subsection{ROC analysis}
\label{resROC}
The Area under the ROC curve (AUC) for all descriptors, along with 95\% confidence intervals, is reported in Tables \ref{tab:resROC} and \ref{tab:resROChn2}. There is a large difference in estimated performance depending on the hard negative mining technique employed, with \textit{hn1} yielding optimistically biased estimates. This is most evident for the CLAIMS dataset, which contains a larger number of visually similar, but not duplicate images. 

For IND detection, the difference between global features and deep learning based features is  {less pronounced. The results are lower than previously reported in literature, because the IND dataset has been significantly expanded, and the new pairs include transformations to which previous descriptors were less robust}.
The DeepRetrieval architecture generally outperforms SPoC for all datasets, despite being trained on a different dataset (Landmarks) with no fine-tuning. The actual gap in performance is very low for the MFND dataset, and increases for other datasets, with CLAIMS exhibiting the highest gap. It is worth noting, however, that DeepRetrieval is also more prone to FPs due to visually similar images, and the performance estimates are more sensitive than for SPoC to the hard negative mining strategy.  {The R-MAC descriptor performs slightly better than SPoC for the MFND dataset, and slightly worse for the other datasets.}

The performance of the SPoC descriptor strongly depends on the network architecture, with Residual Networks consistently outperforming VGG on all datasets. The dataset on which the network was trained has instead a limited impact, possibly due to the effect of PCA whitening. ROC curves for selected descriptors and datasets are reported in Fig. \ref{fig:ROCcurves}. The remaining ROC curves are available as supplementary material. On the MFND collection, the best performance is obtained by the DeepRet descriptor, retrieving 96\% of the true positives at a FP rate of $1.43 \times 10  ^{-6}$.

\begin{figure*} [ht!]
    \centering
   
     \begin{subfigure}[]
     \centering
     \includegraphics[width=0.4\linewidth]{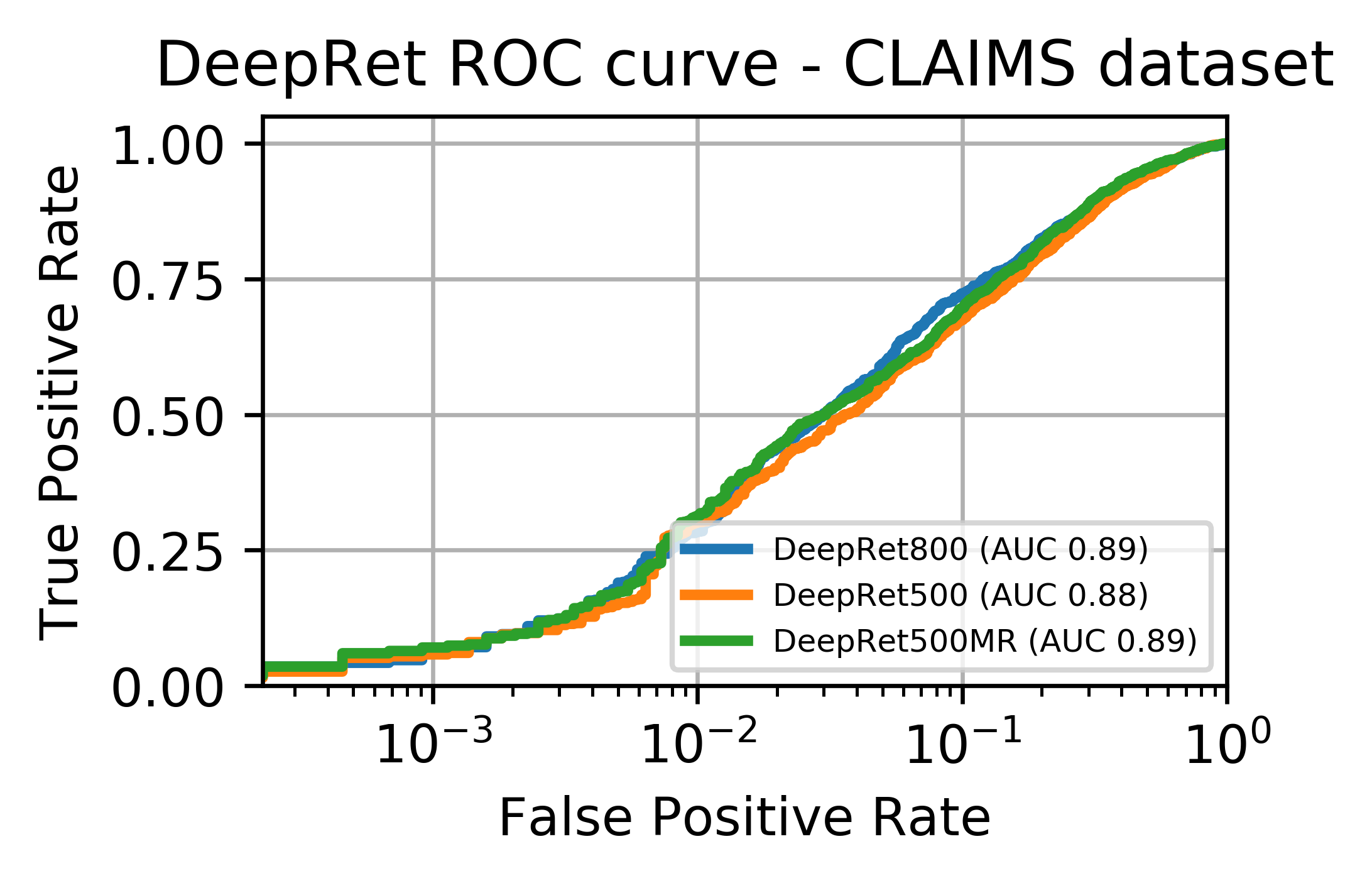}
     \end{subfigure}
        \begin{subfigure}[]
     \centering
     \includegraphics[width=0.4\linewidth]{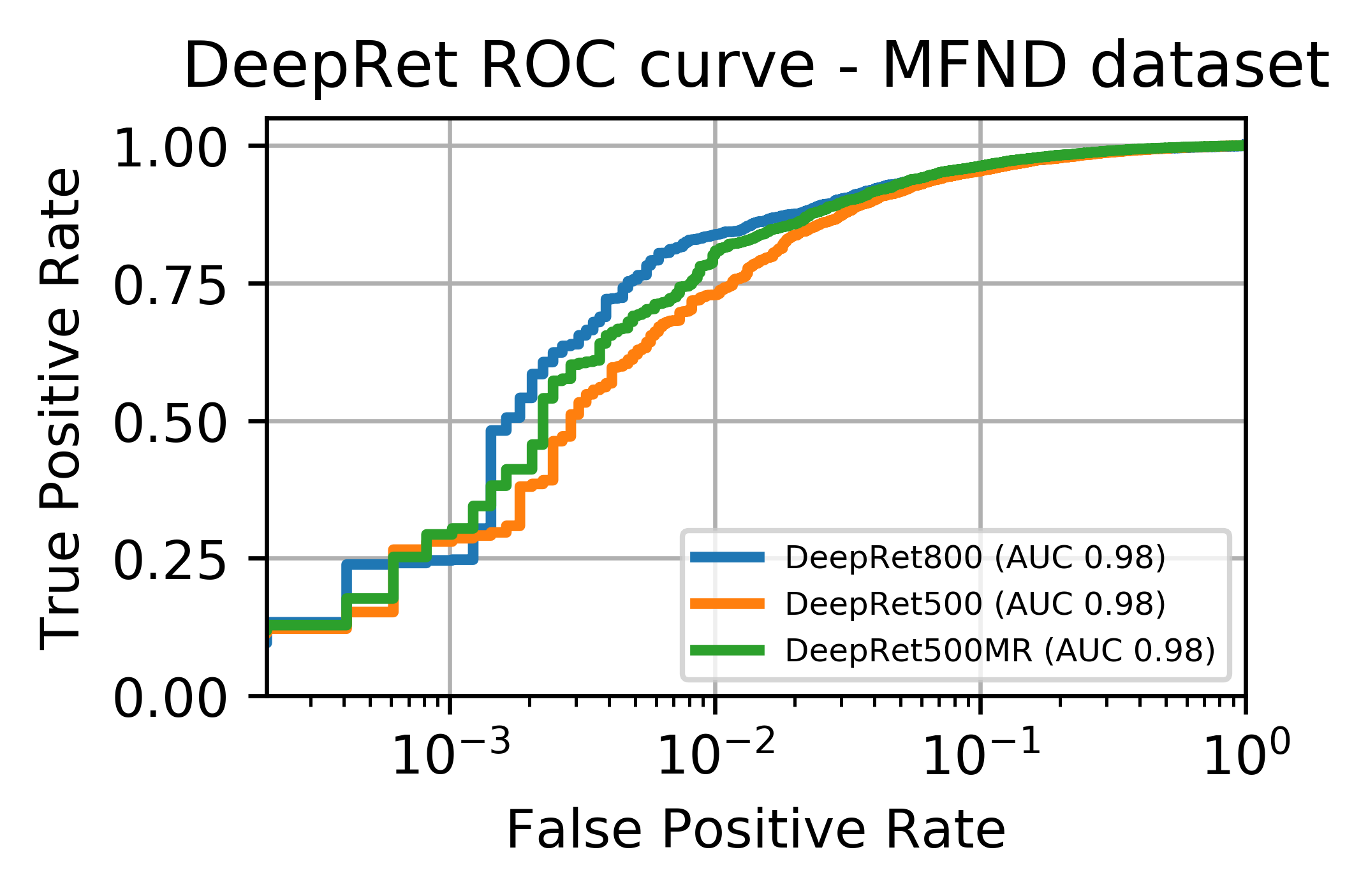}
     \end{subfigure}
     \\
          \begin{subfigure}[]
    
     \includegraphics[width=0.4\linewidth]{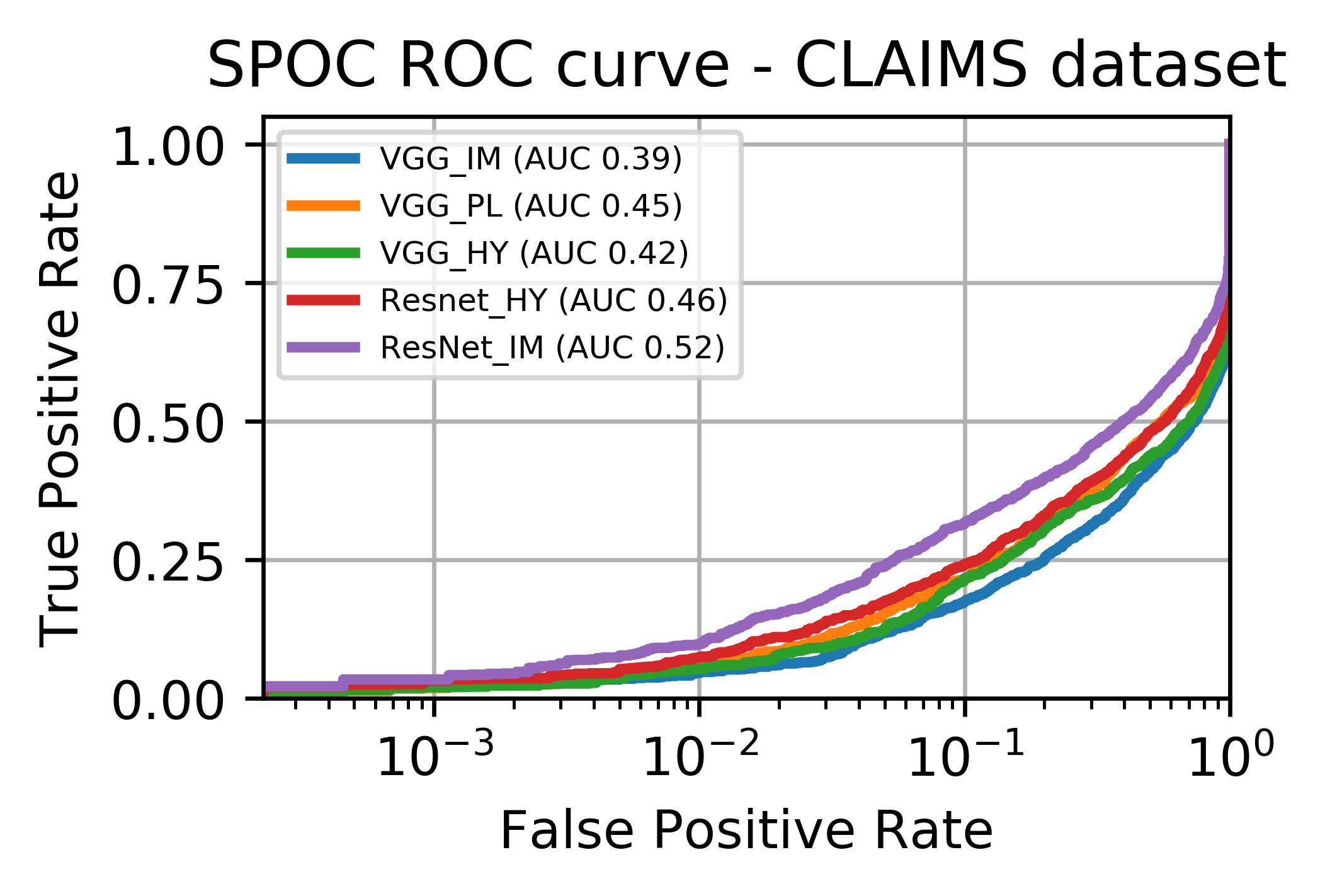}
     \end{subfigure}
          \begin{subfigure}[]
   
     \includegraphics[width=0.4\linewidth]{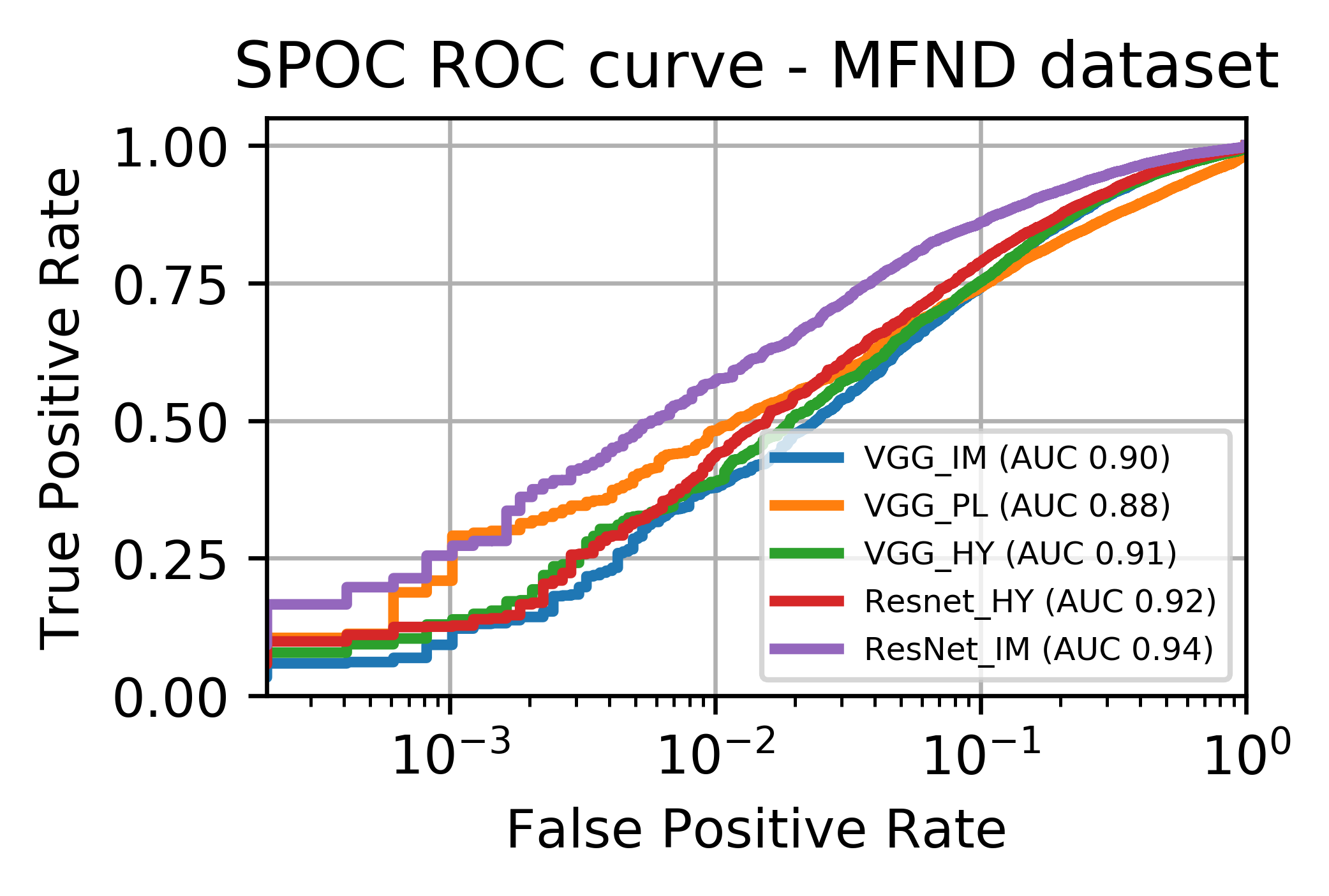}
     \end{subfigure}
     \\
          \begin{subfigure}[]
     
     \includegraphics[width=0.4\linewidth]{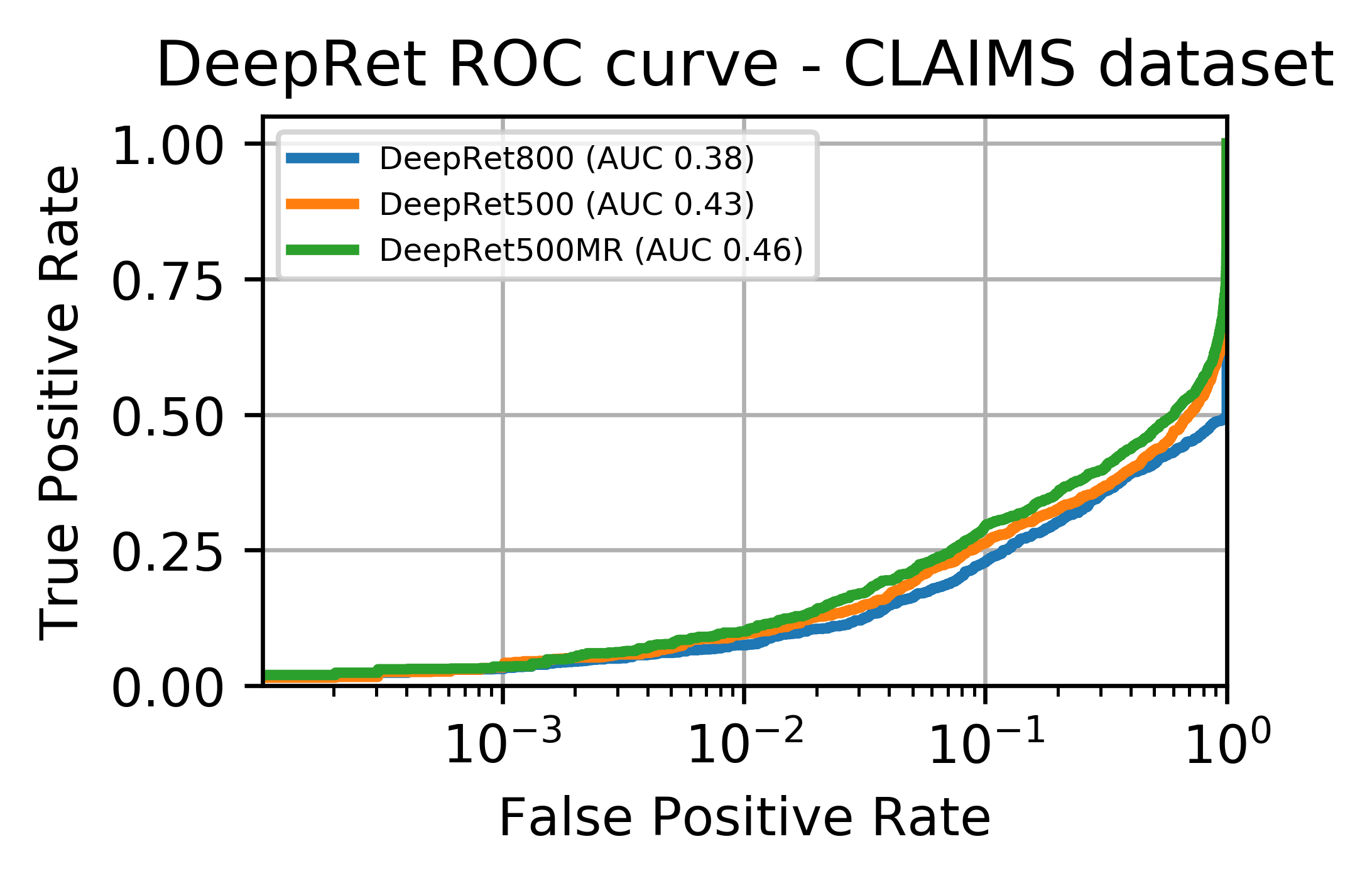}
     \end{subfigure}
          \begin{subfigure}[]
    
     \includegraphics[width=0.4\linewidth]{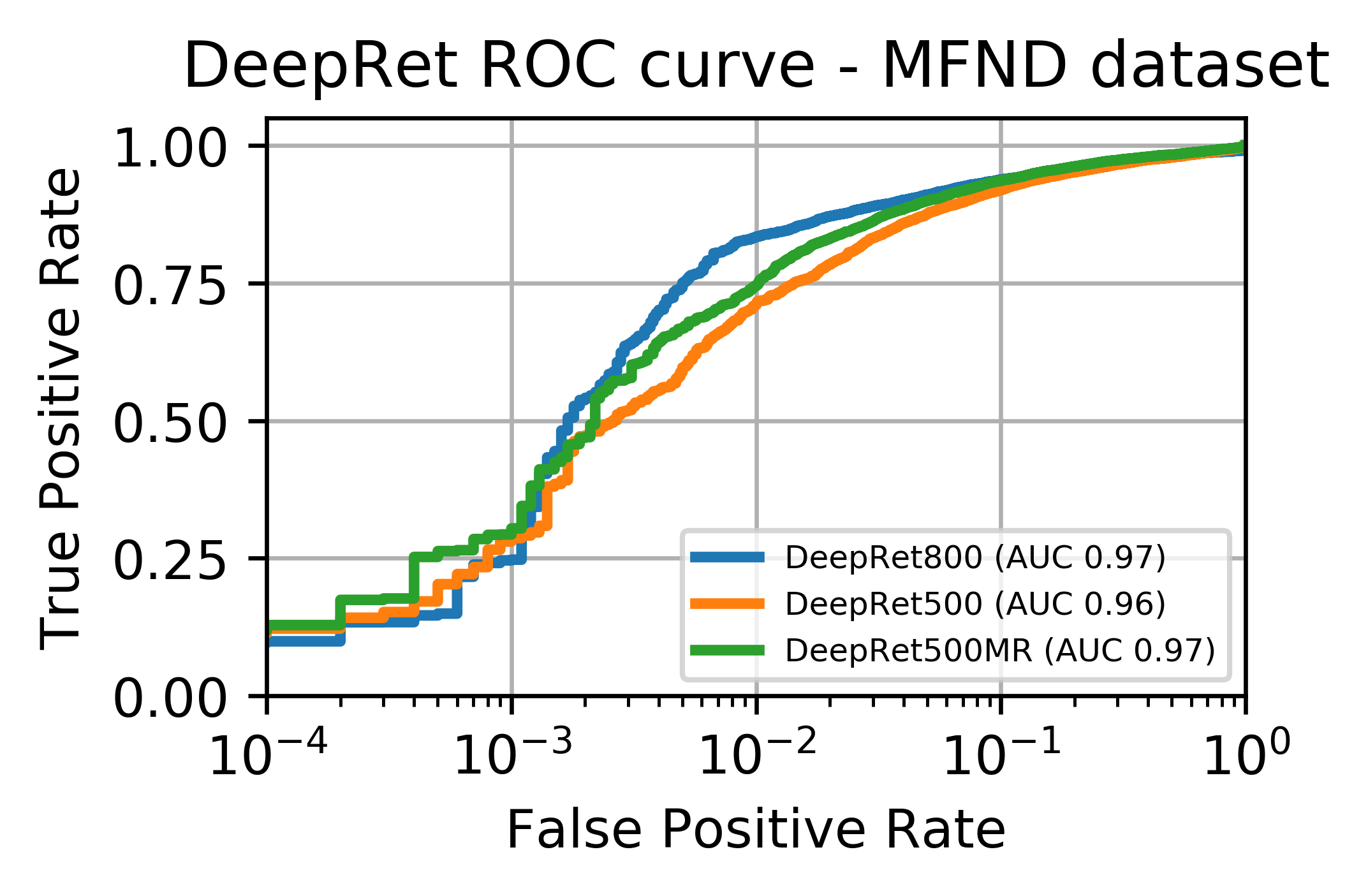}
     \end{subfigure}
     \\
          \begin{subfigure}[]
    
     \includegraphics[width=0.4\linewidth]{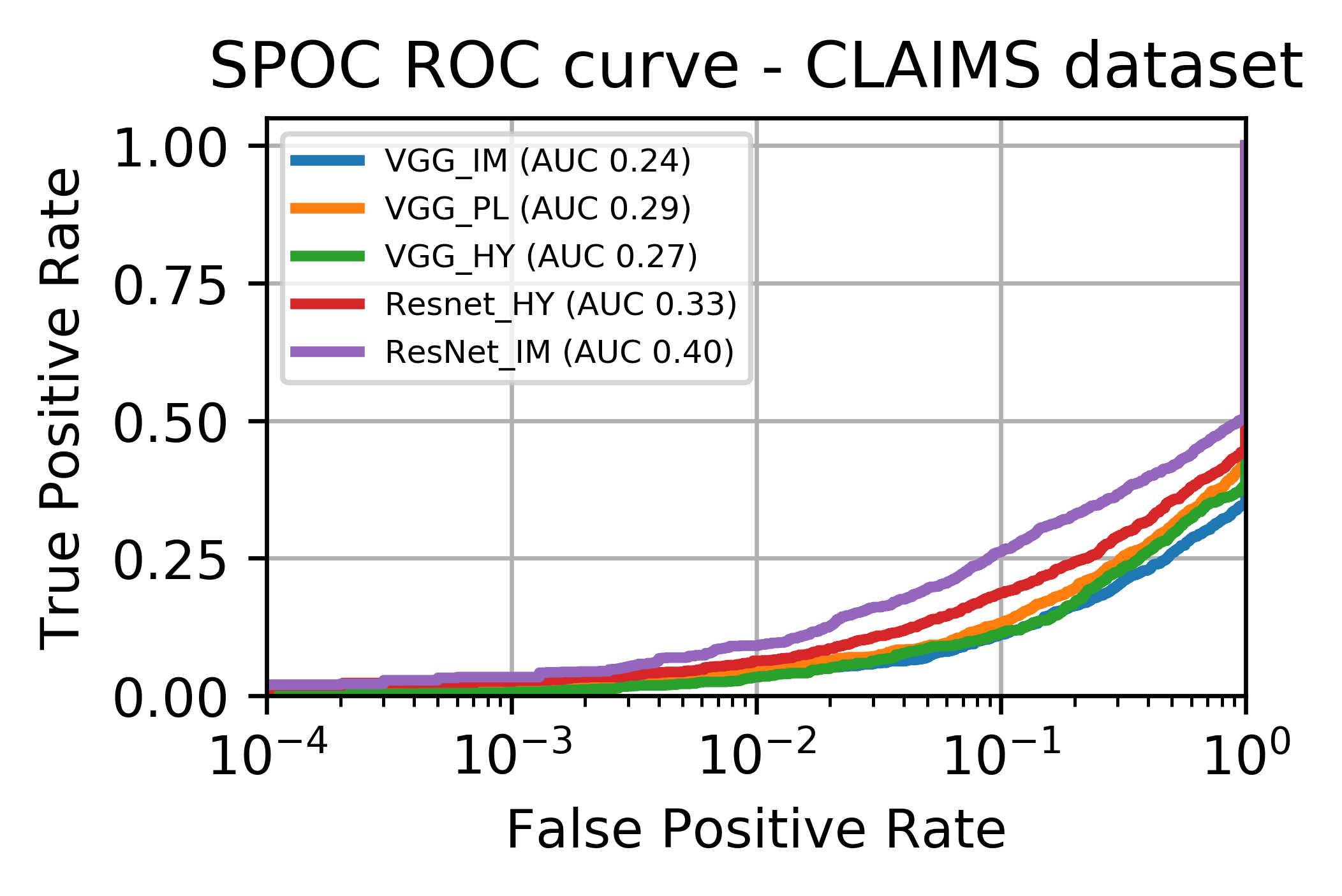}
     \end{subfigure}
          \begin{subfigure}[]
 
     \includegraphics[width=0.4\linewidth]{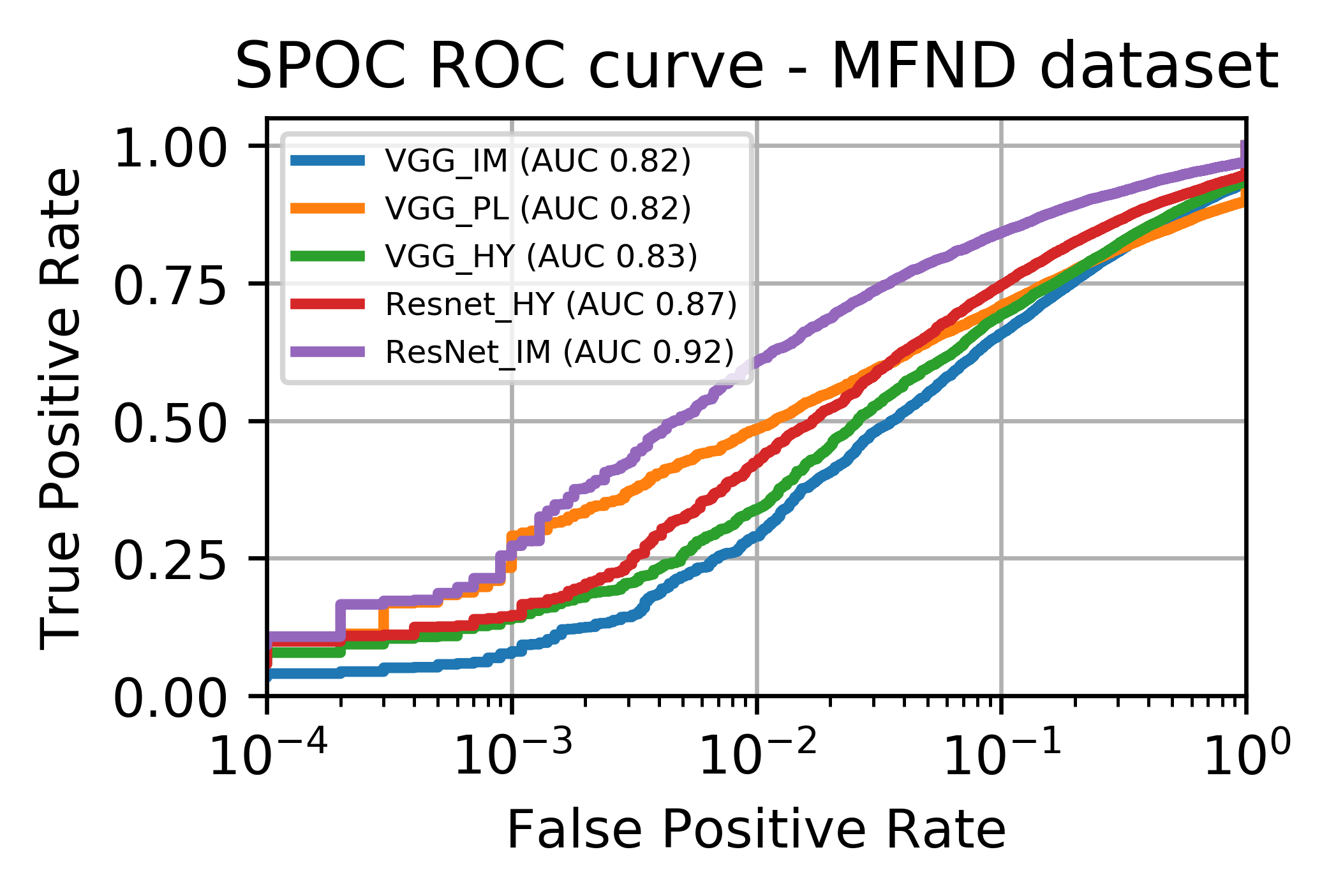}
     \end{subfigure}
     
  \caption{ROC curves for the Deep Retrieval descriptors for hard negative mining strategy \textit{hn1} (a-d) and \textit{hn2} (e-h) respectively. A logarithmic scale was used for the FP rate axis to highlight low values in the 0.01 -- 0.1 range. Since the NND pairs were extracted using a hard negative mining strategy, a 0.1 FP rate corresponds to a projected minimum FP of $1.25 \times 10  ^{-6}$ and $1.43 \times 10  ^{-6}$ for the CLAIMS and MFND datasets, respectively. }
  \label{fig:ROCcurves} 
\end{figure*}

\begin{figure} [ht!]
    \centering
    
    \begin{subfigure}[]
     \centering
     \includegraphics[width=0.4\linewidth]{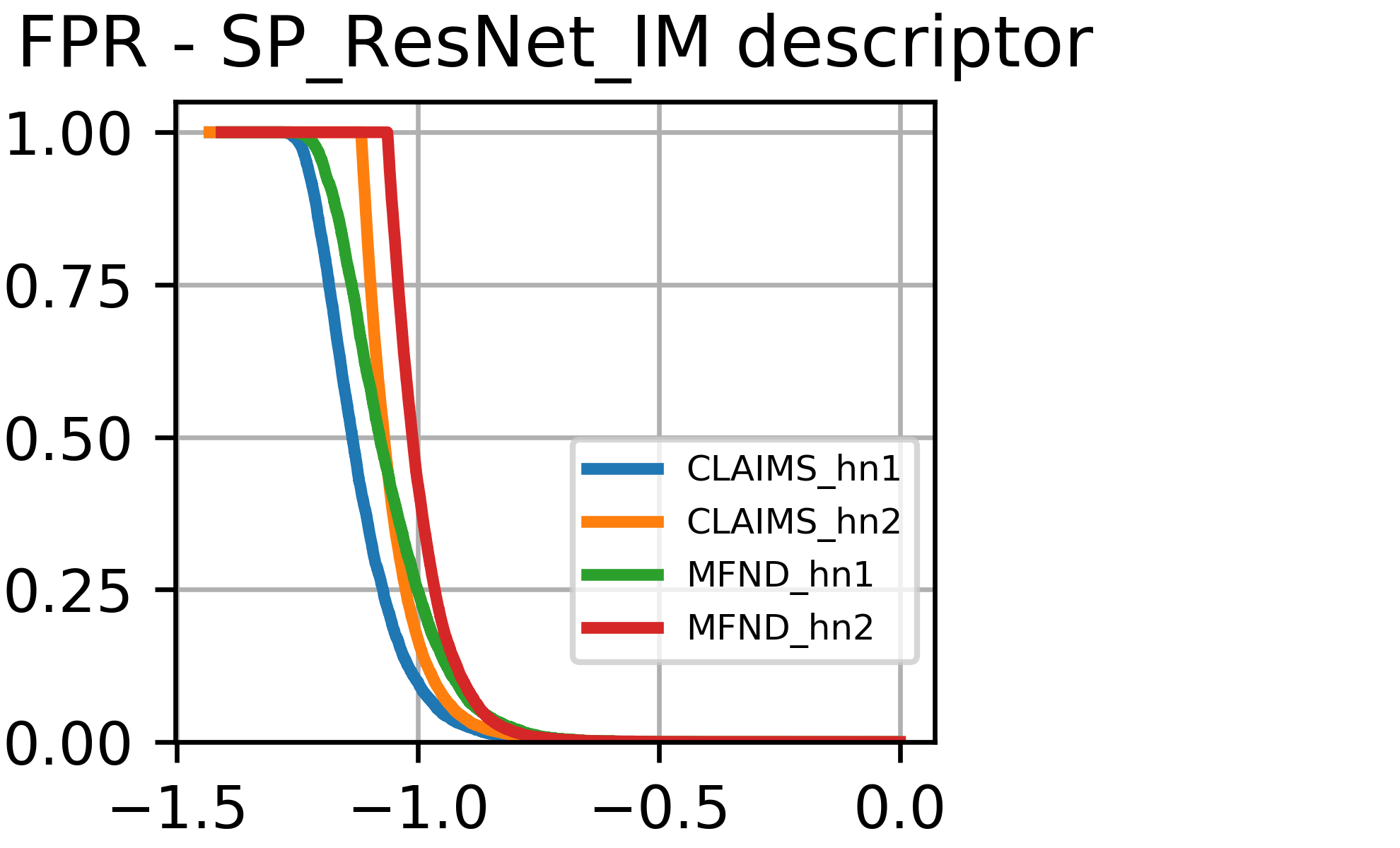}
     \end{subfigure}
      \begin{subfigure}[]
     \centering
     \includegraphics[width=0.4\linewidth]{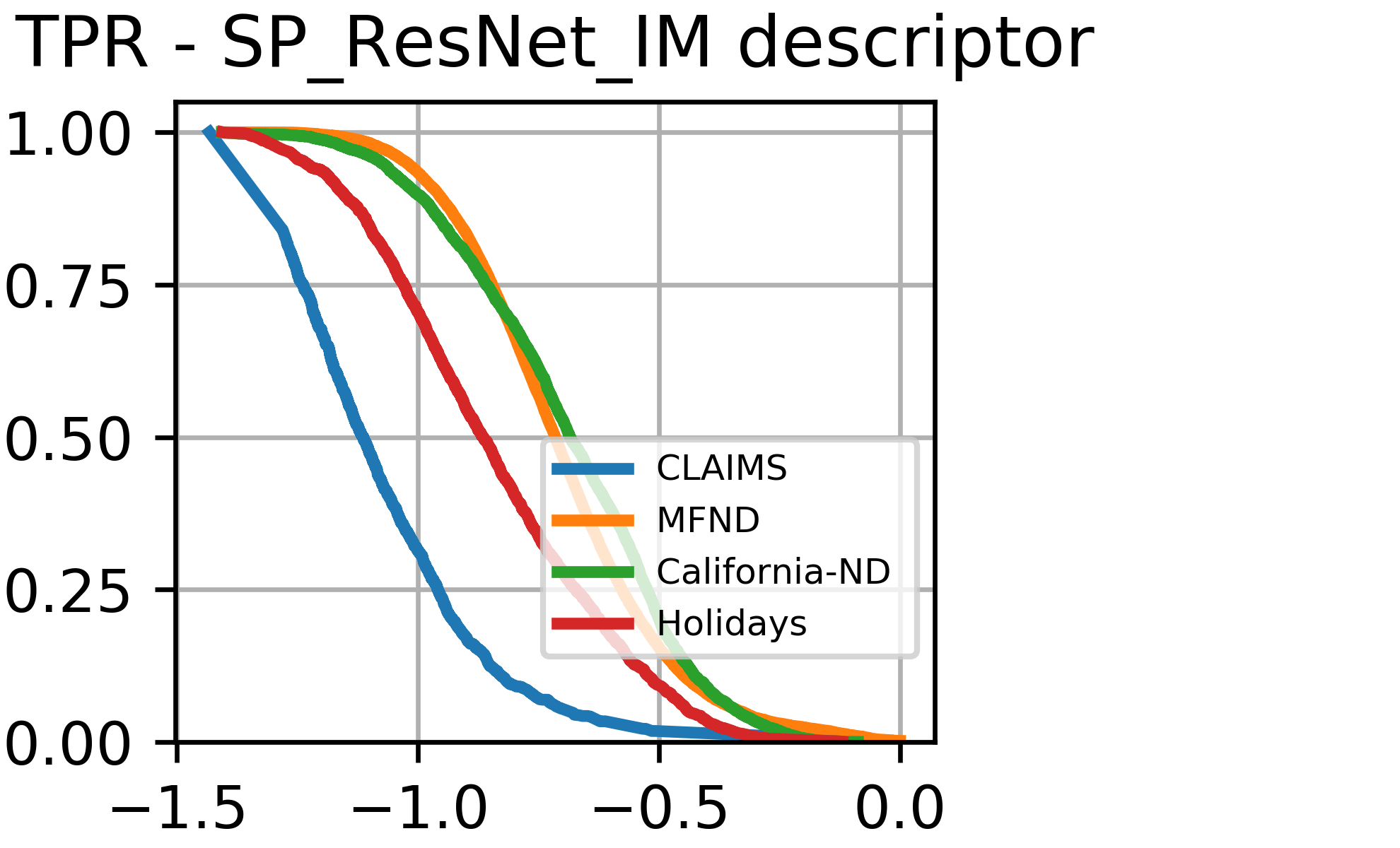}
     \end{subfigure}
     \\
    \begin{subfigure}[]
     \centering
        \includegraphics[width=0.4\linewidth]{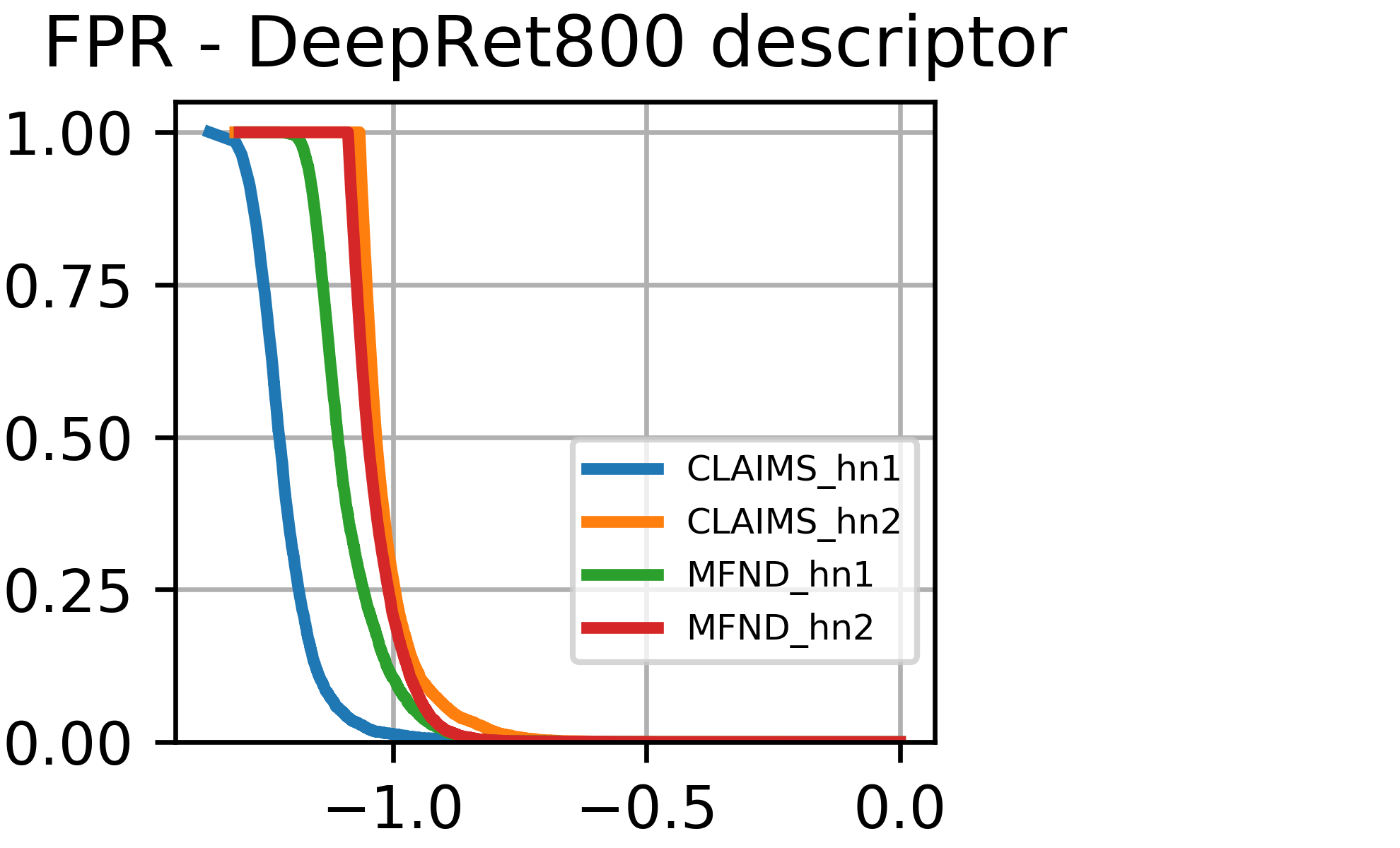}
     \end{subfigure}
    \begin{subfigure}[]
     \centering
      \includegraphics[width=0.4\linewidth]{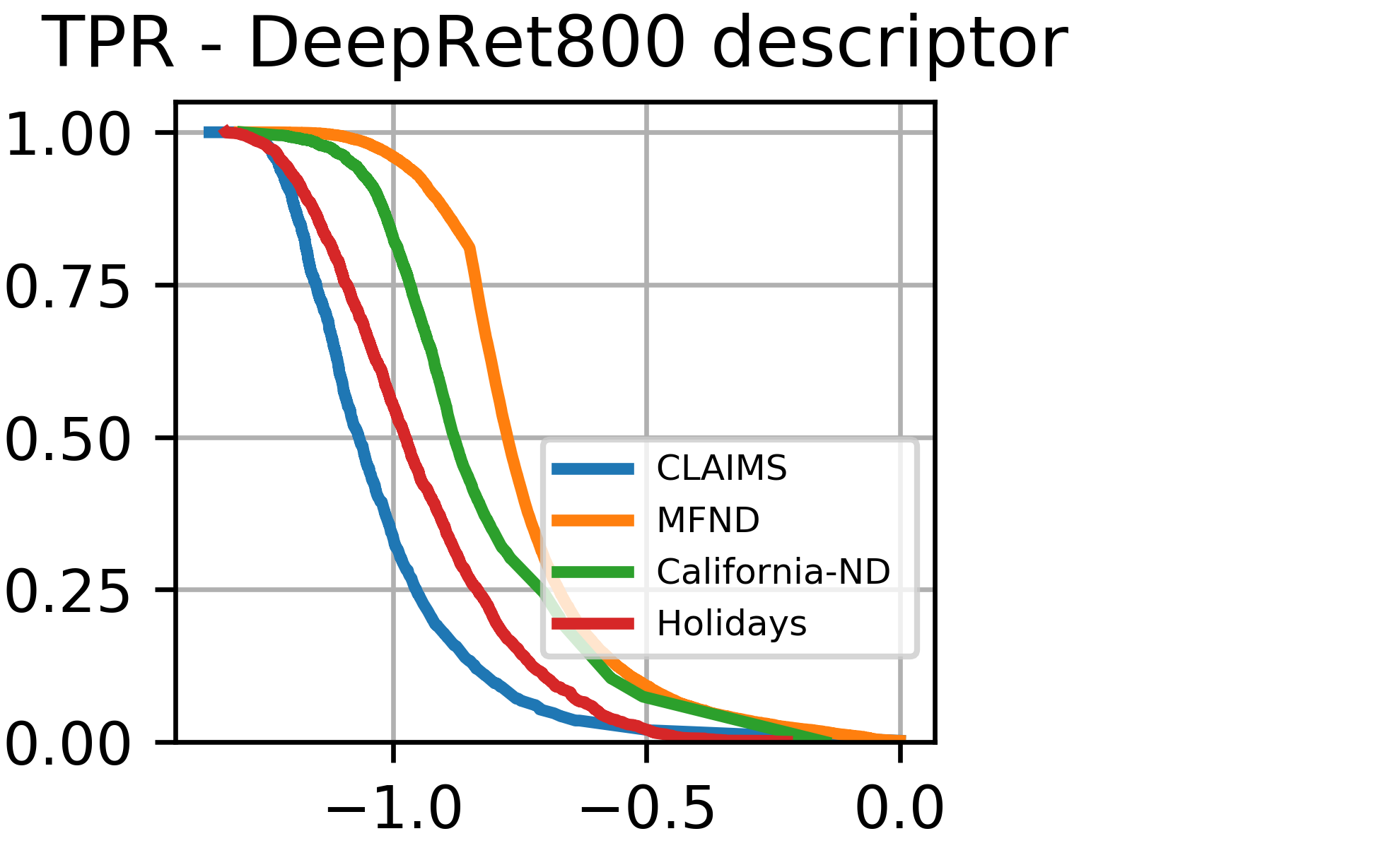}
     \end{subfigure}
  
  \caption{Comparison of TP rate and FP rates across different datasets for selected descriptors: SP\_ResNet101\_ImageNet (a-b), and DeepRet800 (c-d). For FPs, pairs selected with both hard negative mining strategies \textit{hn1} and \textit{hn2} are separately plotted.  }
  \label{fig:TPR-FPR} 
\end{figure}

\subsection{Dataset comparison}
\label{datasetcomp}
In order to better highlight differences between the datasets, we computed the false FP and TP rate w.r.t. the distance threshold for each dataset and for each of the two best performing descriptors, as detailed in Fig. \ref{fig:TPR-FPR}. 

Not surprisingly, INDs are more easily detected than NINDs. The CLAIMS dataset contains the most challenging near duplicates, closely followed by the Holidays dataset. Given the annotation procedure followed for the MFND benchmark, it is possible that the NIND examples are skewed towards examples that are more easily detected using the present descriptors, and future experiments will likely find new examples. Examples of ND pairs that were poorly scored are reported in Fig. \ref{fig:exampleCLAIMS}; empirically, large changes in viewpoint appear among the most challenging differences.

We also compared FP rates on the MFND and CLAIMS datasets with the two hard negative mining strategies. For \textit{hn1}, MFND appears to be more difficult than CLAIMS, whereas for \textit{hn2} the two datasets are  {quite} comparable for both descriptors. Given a random query image, it is more likely to find a similar image for MFND than CLAIMS, but CLAIMS contains larger clusters of images that are both semantically and visually similar, as is likely going to be the case for any dataset that comes from a focused domain. Examples of hard negatives (\textit{hn2}) for both datasets are reported in Fig. \ref{fig:exampleCLAIMS}.

\begin{figure*}[ht!]
  \centering
  \begin{subfigure}[]
     \centering
        \includegraphics[width=0.3\linewidth]{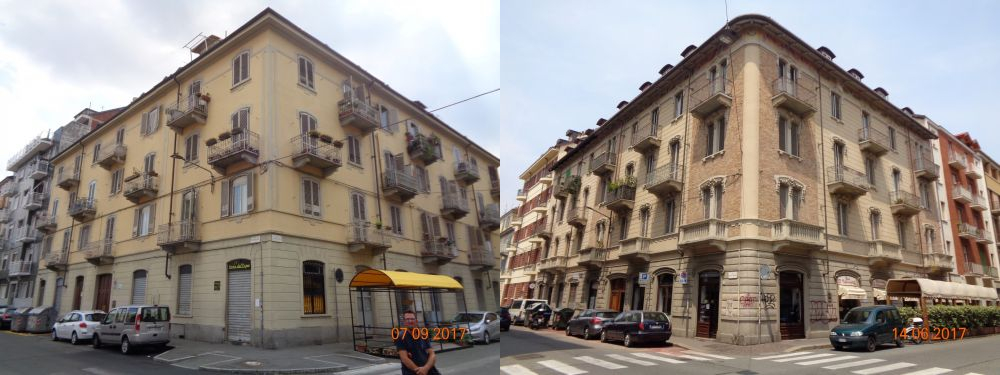}
   \end{subfigure}
    \begin{subfigure}[]
     \centering
        \includegraphics[width=0.3\linewidth]{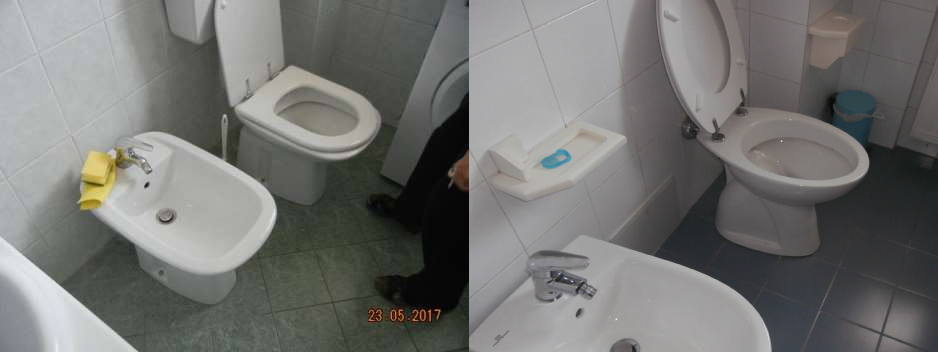}
   \end{subfigure}
     \begin{subfigure}[]
     \centering
        \includegraphics[width=0.3\linewidth]{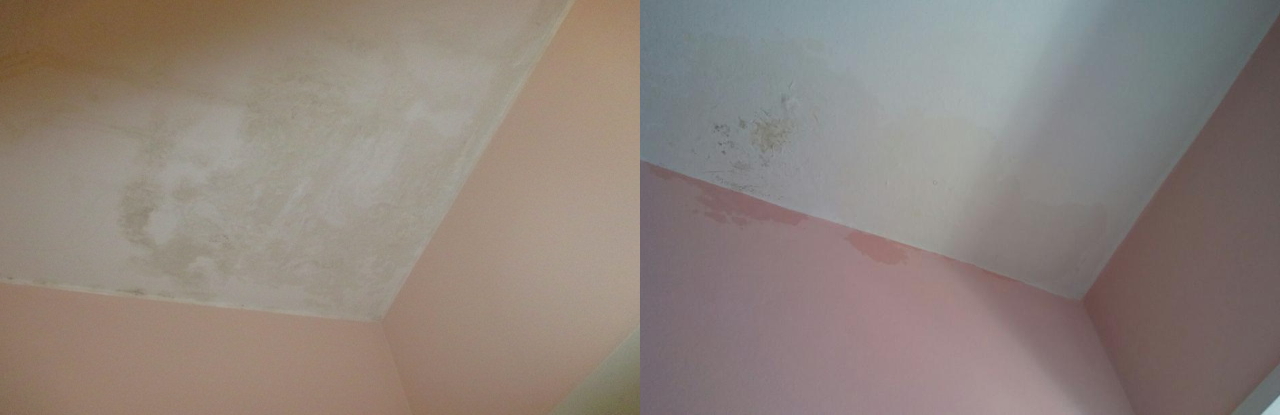}
   \end{subfigure}
   \\
   \begin{subfigure}[]
     \centering
        \includegraphics[width=0.3\linewidth]{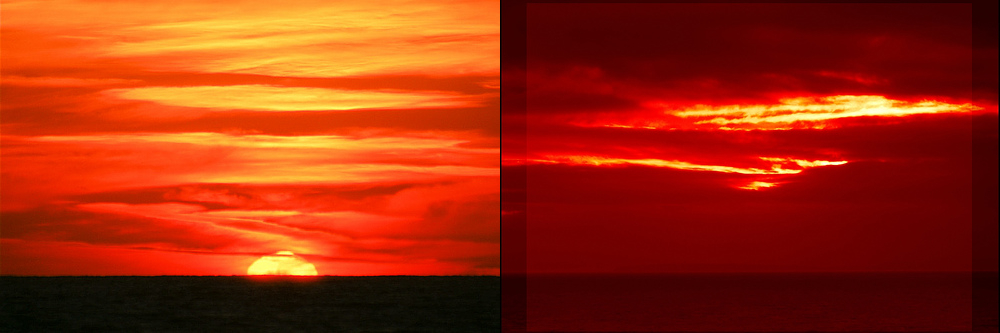}
   \end{subfigure}
    \begin{subfigure}[]
     \centering
        \includegraphics[width=0.3\linewidth]{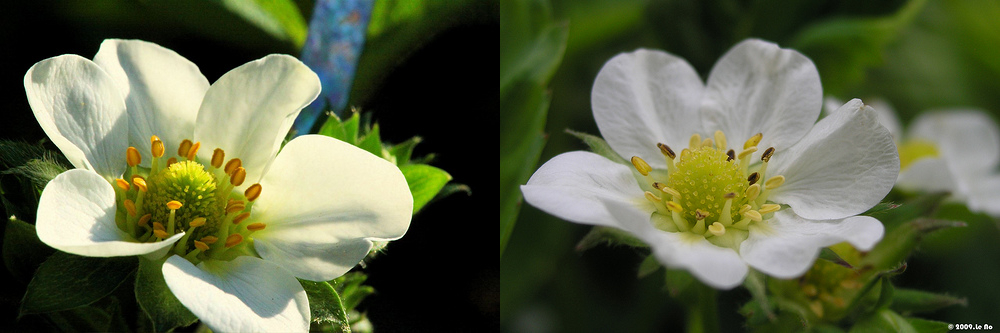}
   \end{subfigure}
     \begin{subfigure}[]
     \centering
        \includegraphics[width=0.3\linewidth]{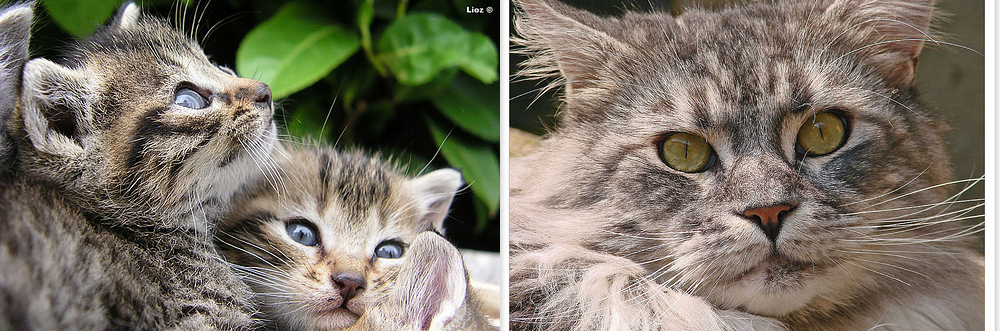}
   \end{subfigure}
   \\
   \begin{subfigure}[]
     \centering
        \includegraphics[width=0.3\linewidth]{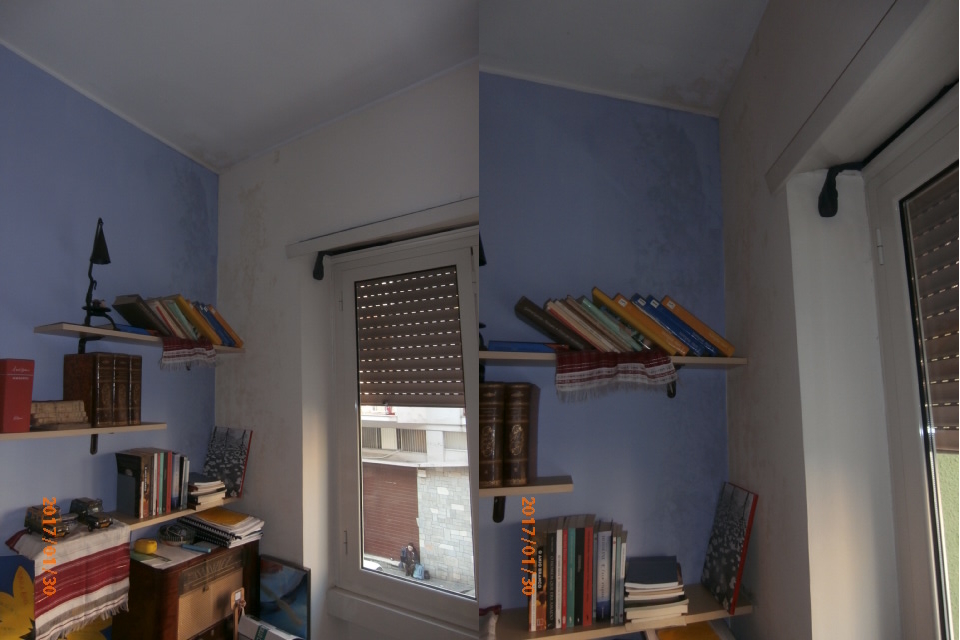}
   \end{subfigure}
    \begin{subfigure}[]
     \centering
        \includegraphics[width=0.3\linewidth]{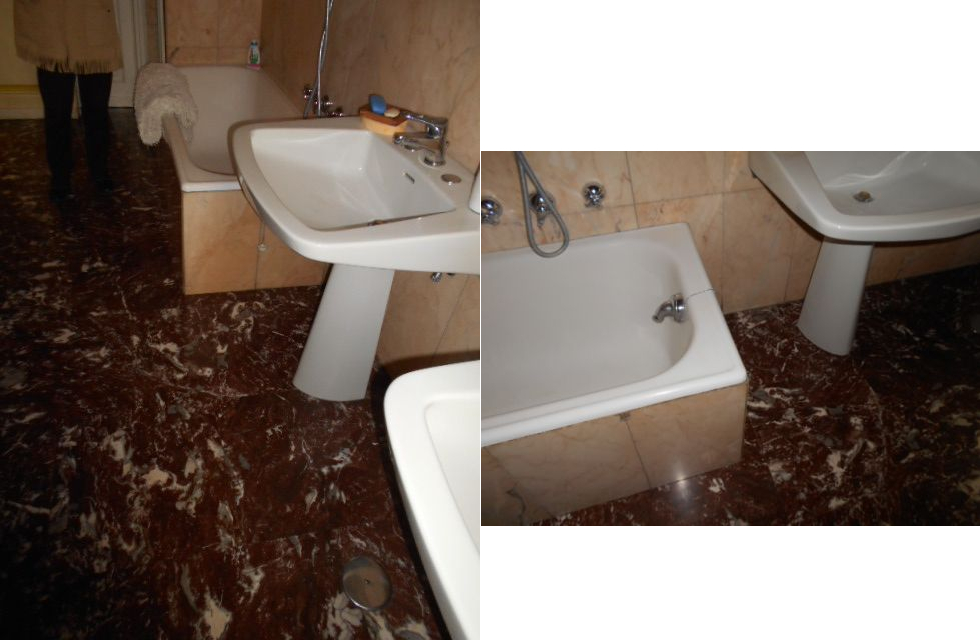}
   \end{subfigure}
     \begin{subfigure}[]
     \centering
        \includegraphics[width=0.3\linewidth]{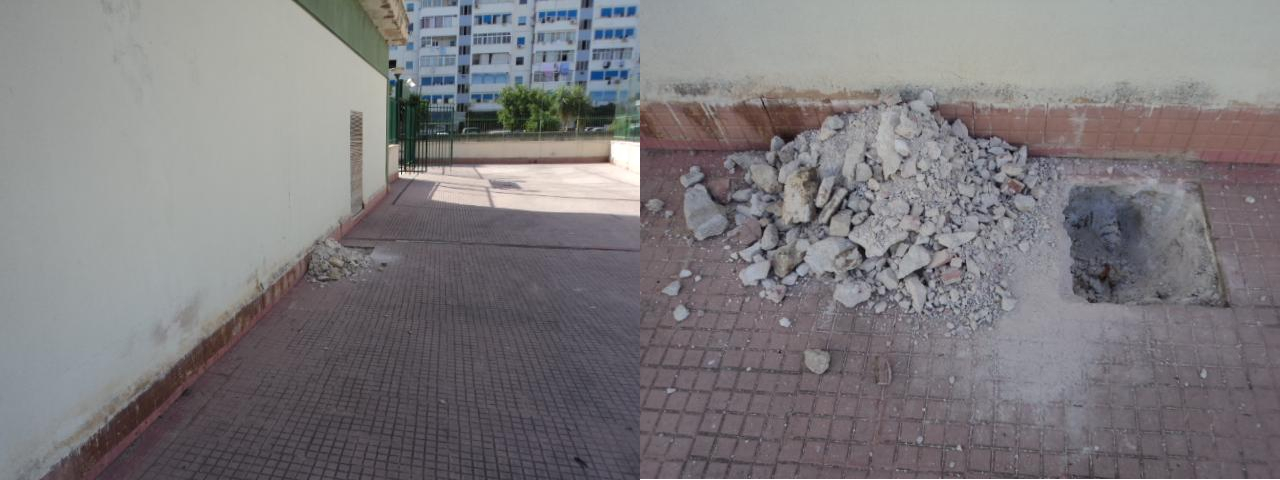}
   \end{subfigure}
   \\

\caption{ {Examples of challenging image pairs for unsupervised near-duplicate detection. Examples (a-e) are challenging negative examples which, despite high semantic and visual similarity, are not near duplicates. Examples derived from CLAIMS (a-c) are related to image types that are particularly common this collection, whereas examples from MFND (d-e) are mostly of subjects which are particularly popular on Internet, such as sunsets and cats. Examples (g-h) are challenging near-duplicates from the CLAIMS dataset which were given low similarity scores by all descriptors; common patterns that are difficult to detect include drastic changes in viewpoint, or one of the two images in the pair represents a detail of a larger scene.}}
\label{fig:exampleCLAIMS}
\end{figure*}

\subsection{Query performance analysis}
\label{perfquery}

\begin{figure*}[ht!] 
    \centering
     \begin{subfigure}[]
     \centering
        \includegraphics[width=0.4\linewidth]{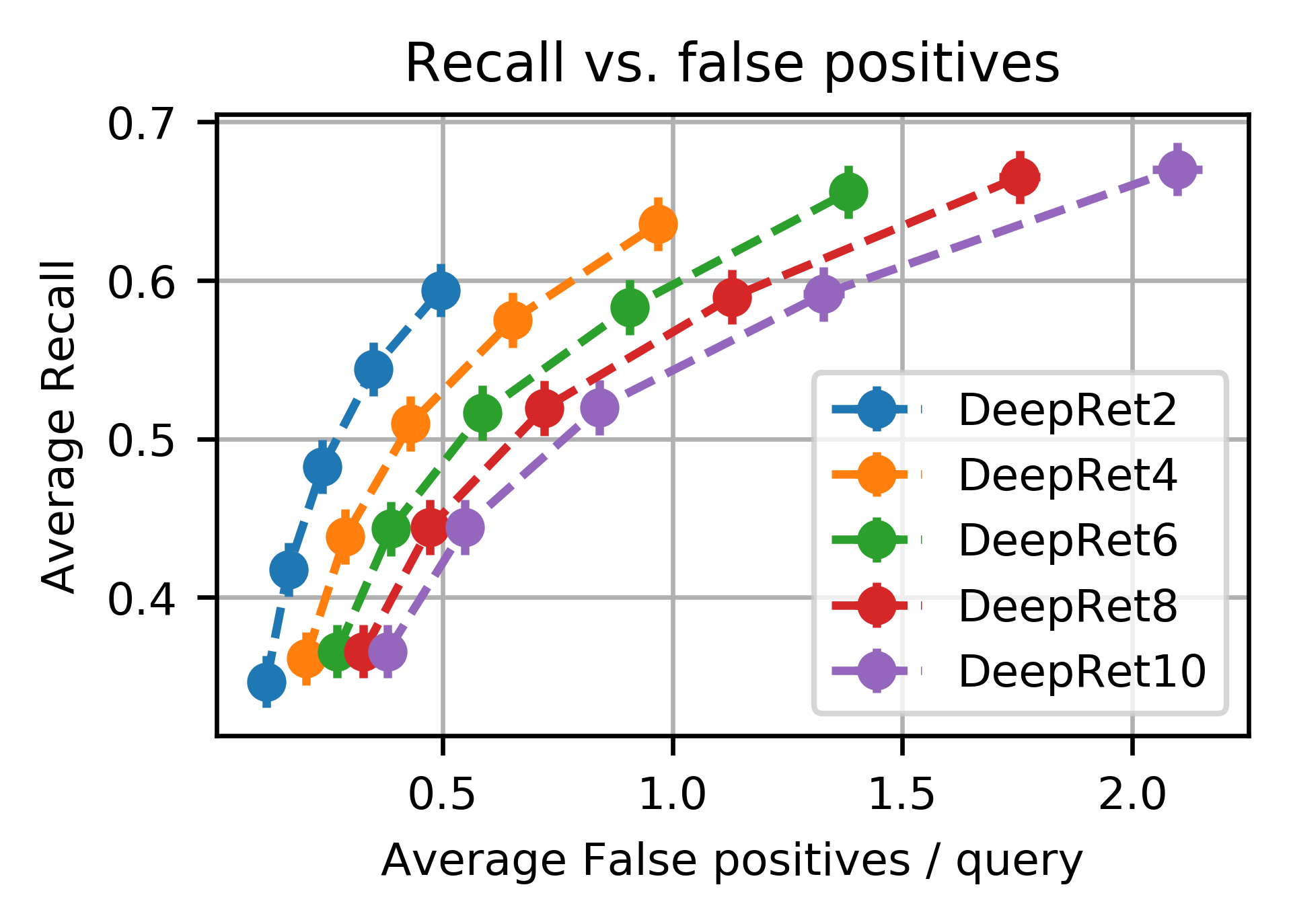}
   \end{subfigure}
    \begin{subfigure}[]
     \centering
        \includegraphics[width=0.4\linewidth]{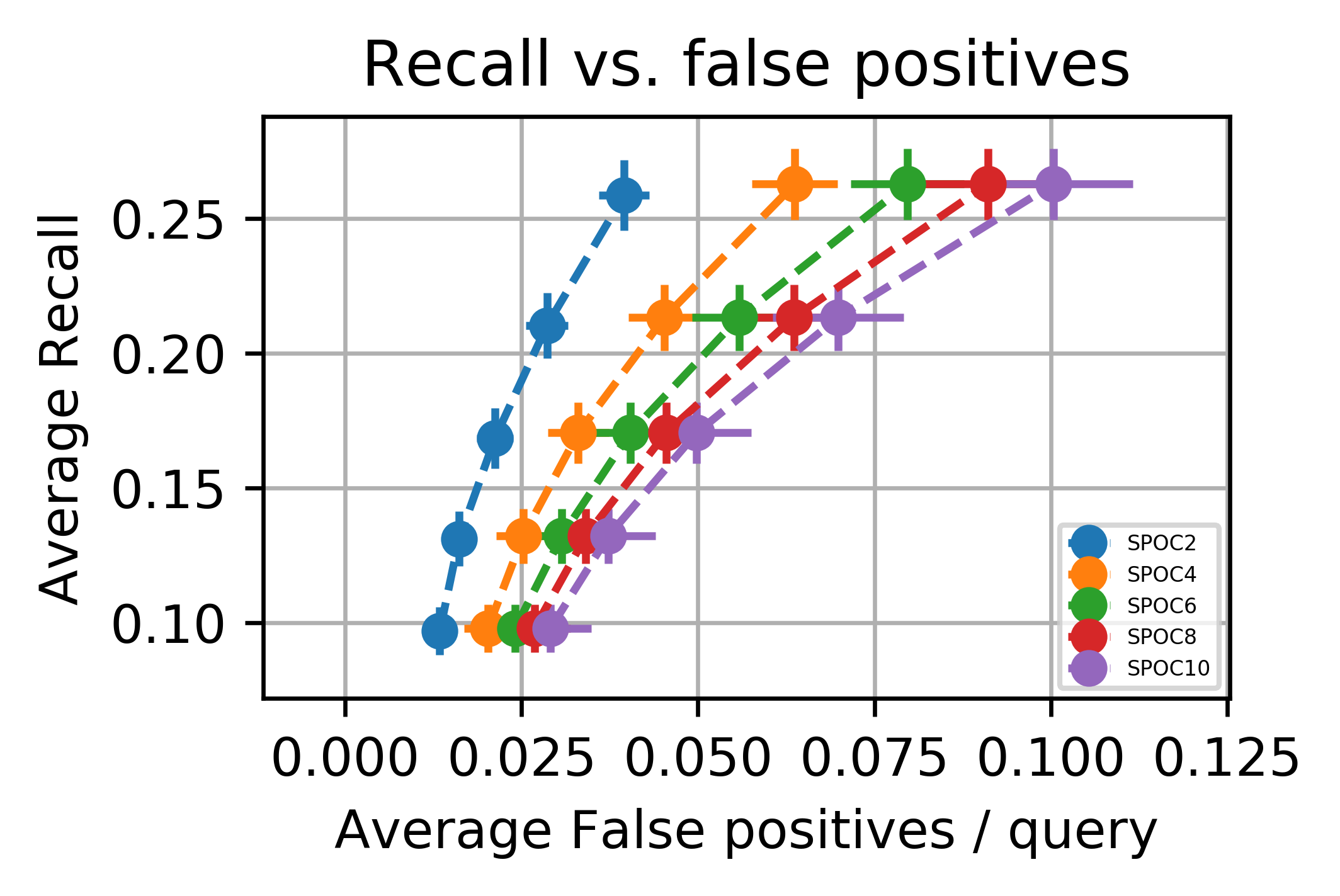}
   \end{subfigure}
   \\
     \begin{subfigure}[]
     \centering
        \includegraphics[width=0.4\linewidth]{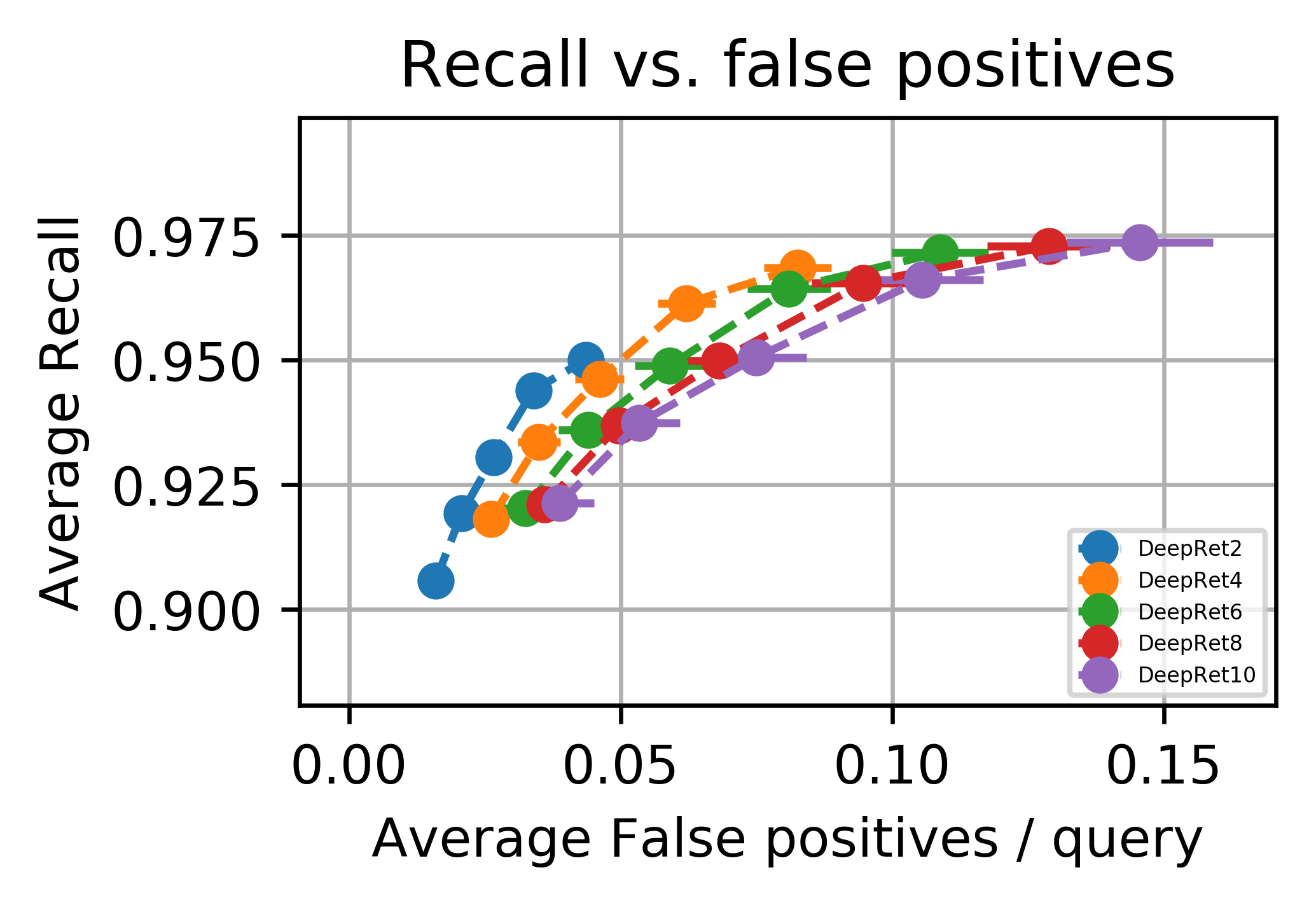}
   \end{subfigure}
  \begin{subfigure}[]
     \centering
        \includegraphics[width=0.4\linewidth]{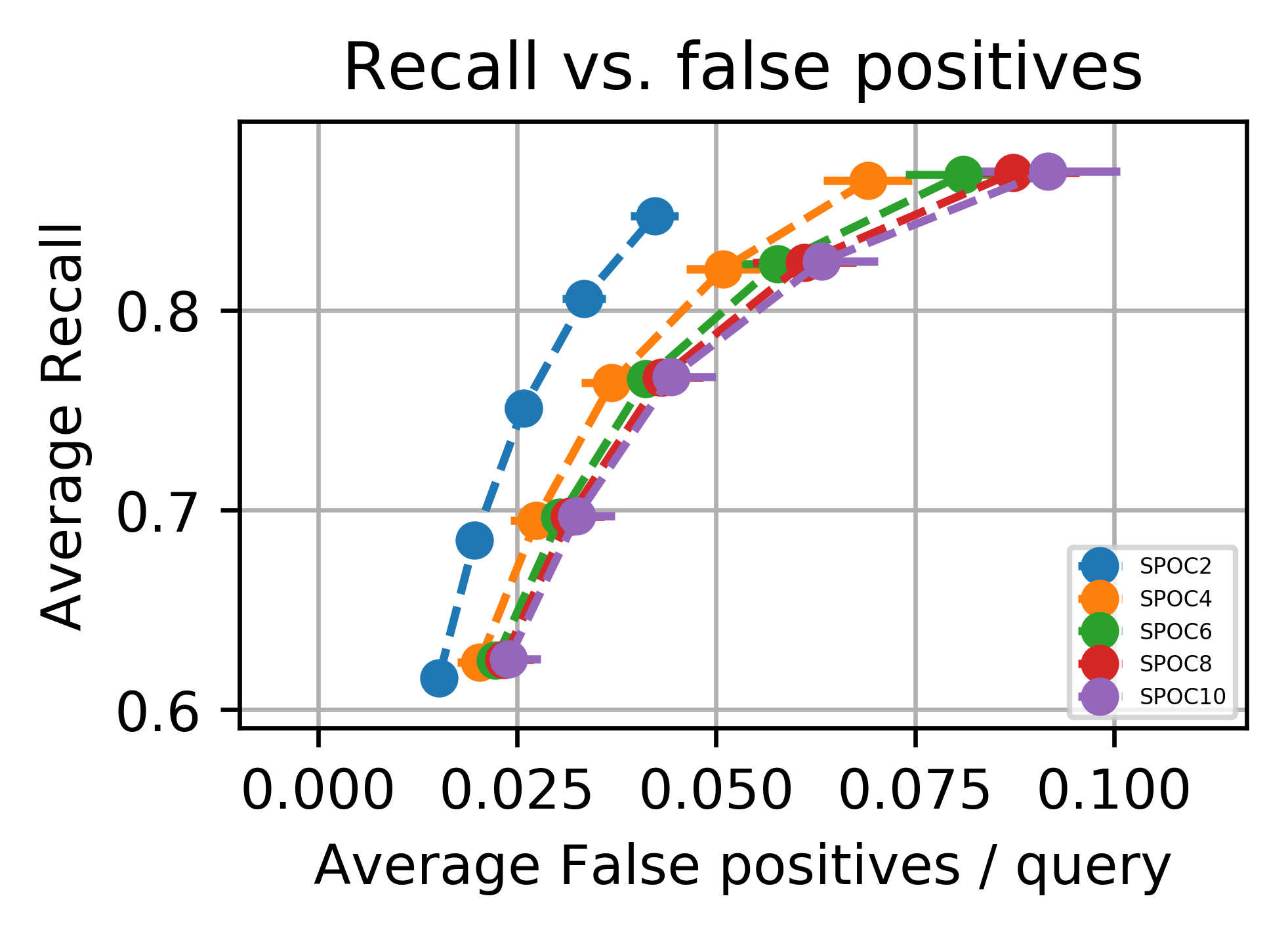}
   \end{subfigure}
   \\
   \centering
     \begin{subfigure}[]
     \centering
        \includegraphics[width=0.4\linewidth]{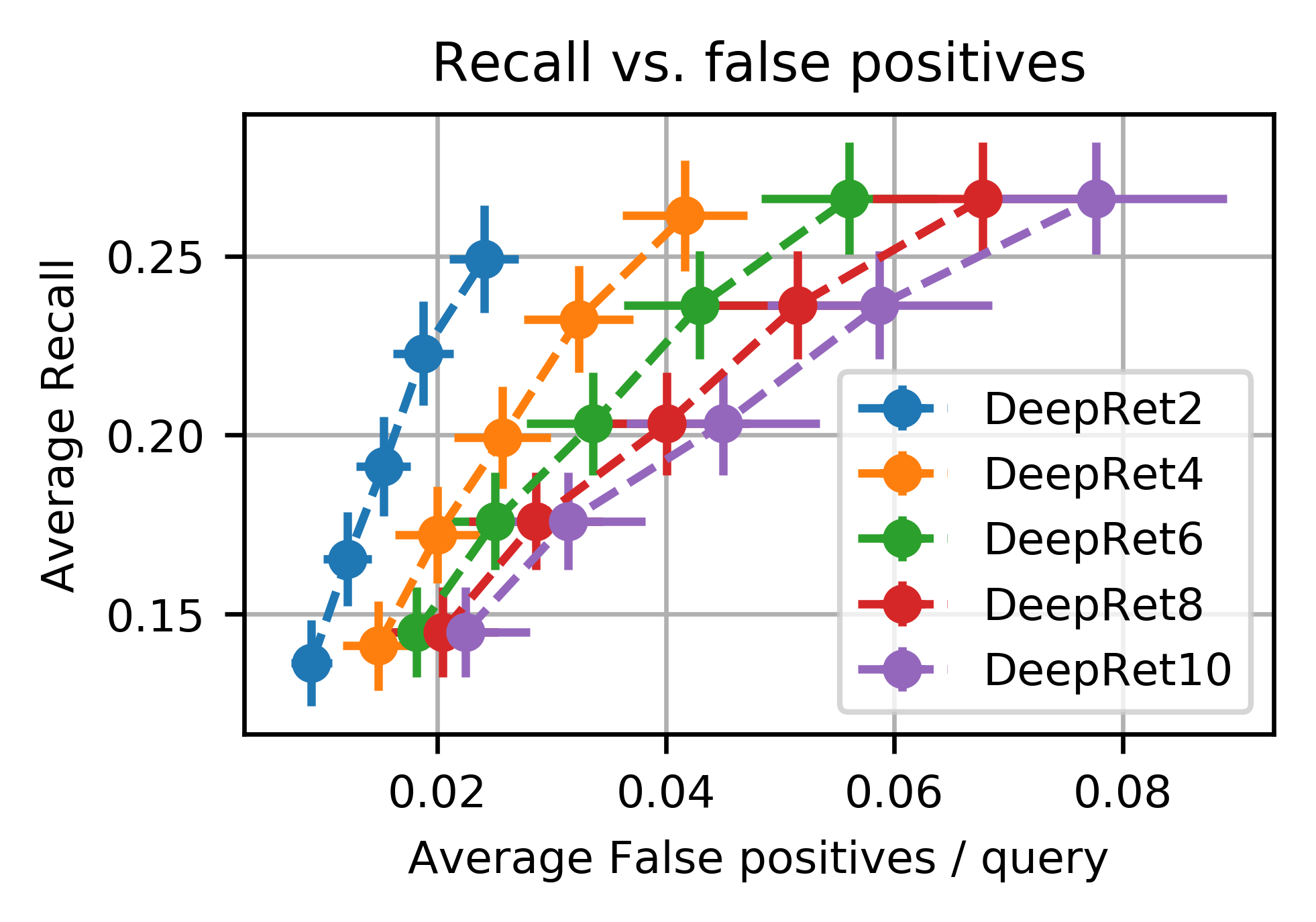}
   \end{subfigure}
    \begin{subfigure}[]
     \centering
        \includegraphics[width=0.4\linewidth]{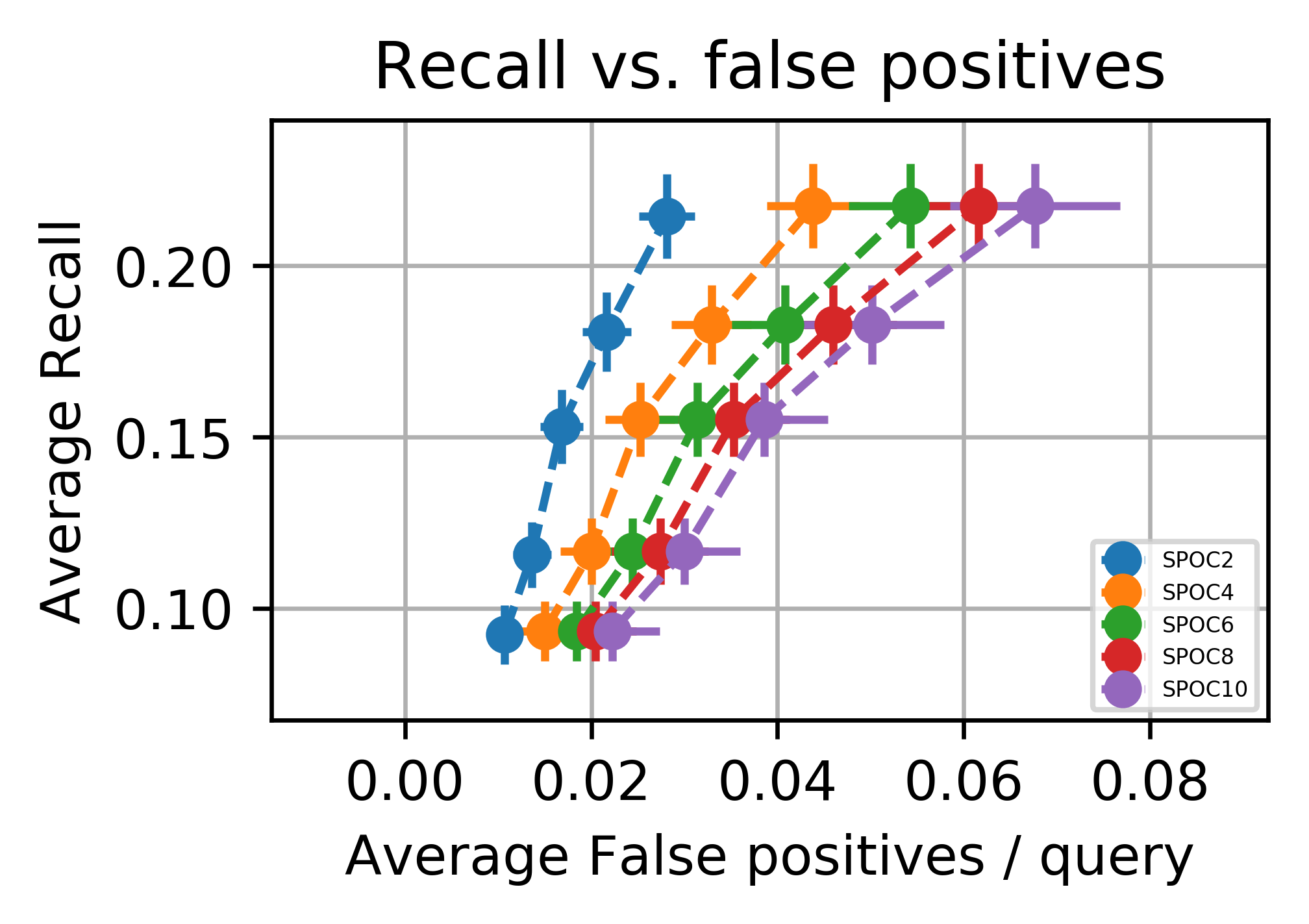}
   \end{subfigure}
   \\
     \begin{subfigure}[]
     \centering
        \includegraphics[width=0.4\linewidth]{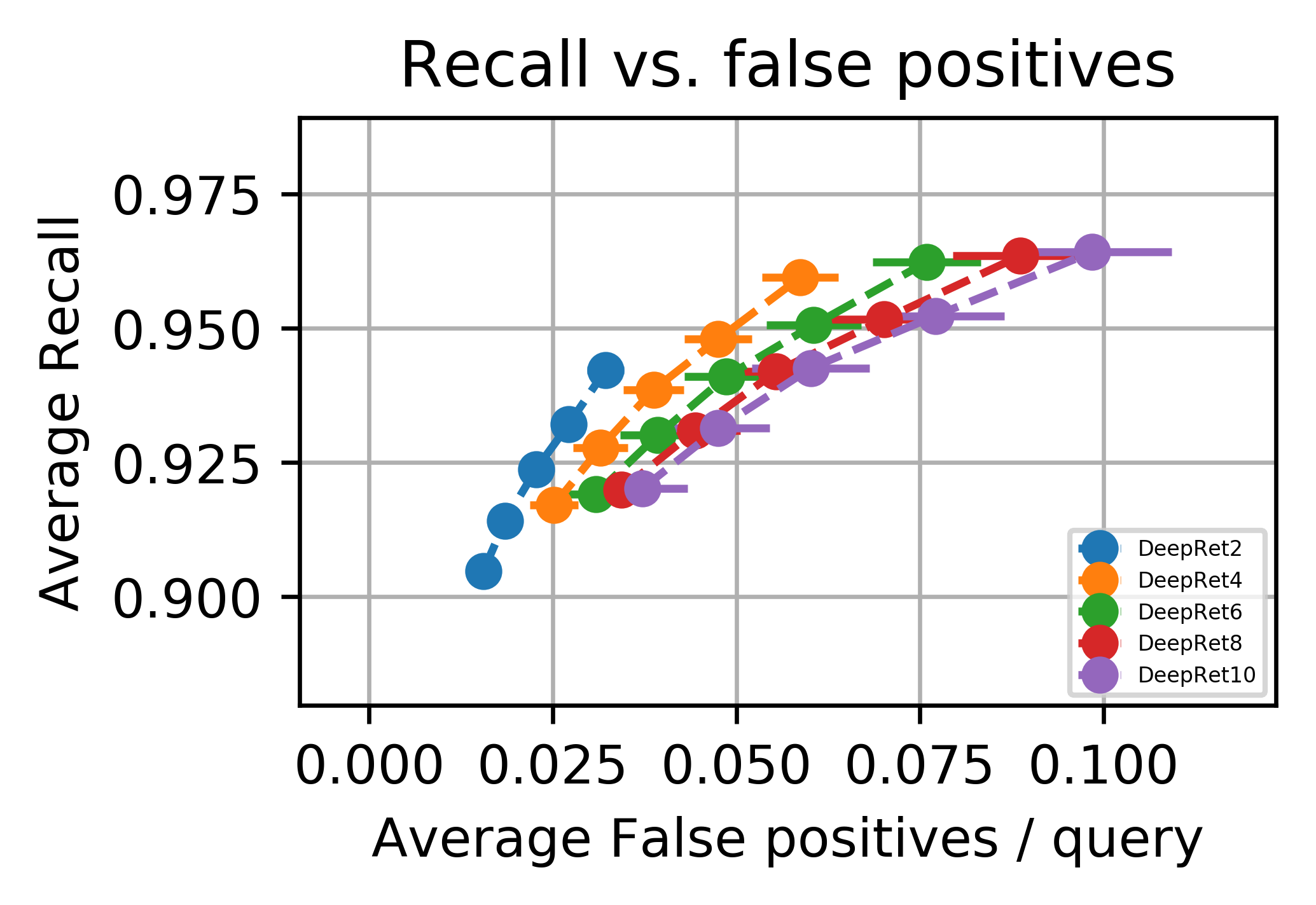}
   \end{subfigure}
     \begin{subfigure}[]
     \centering
        \includegraphics[width=0.4\linewidth]{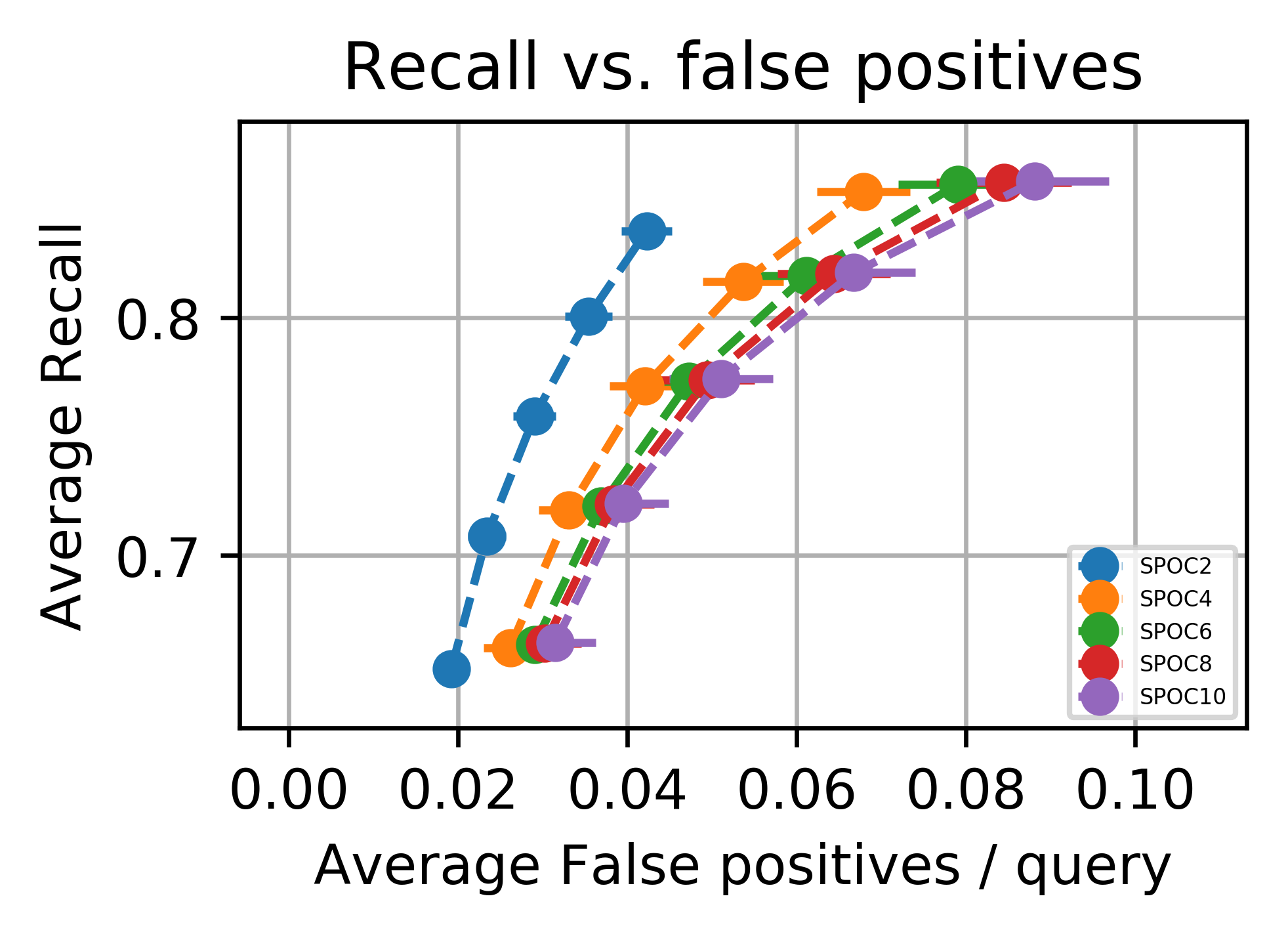}
   \end{subfigure}

  \caption{Average recall vs. FPs/query for the CLAIMS (a-b, e-f) and MFND dataset (c-d, g-h), with thresholds calculated using hard negative mining \textit{hn1} (top row, a-d) and \textit{hn2} (bottow row, e-h). Performance is measured at fixed thresholds (dots in the above curves), bars indicate the standard error. The maximum number of images retrieved by each query is limited to $K=2,4,6,8,10$, results are plotted as separate curves. }
  \label{fig:FROC} 
\end{figure*}

In this section, the two best performing descriptors at ROC analysis were compared: DeepRet800 and SP\_ResNet152IN, using the experimental setup detailed in Section \ref{secFROCanalysis}.  

We performed threshold-limited queries at thresholds $T$ corresponding to a FP rate in the [0.01 -- 0.1] range, and a maximum number of results/query $K$ between 2 and 10. 

The results are plotted in Fig. \ref{fig:FROC}. Since we are using the same dataset for both hard negative mining and estimating query performance, it can be easily shown from Eq. \ref{eqFPq} that a FP rate of 0.1 should correspond to an average number of FPs/query of 0.1 as well. Estimates based on \textit{hn1} have larger deviations from expected values, especially for the DeepRet descriptor on the CLAIMS dataset, which is a 20x larger than expected.  {For \textit{hn2}, a}ctual measured FPs/query are usually slightly lower than predicted.  {Since we limit the maximum number of images retrieved by each query, this factor may explain the discrepancy, which is higher for the CLAIMS dataset where images are more tightly clustered in feature space. It should also be noticed that in Eq. \ref{eqFPq} the specificity depends only on the threshold, and not on the query image; our experiments, however, suggest that this does not hold true in practice, and that certain types of images are more prone to false positives.}

\section{Discussion}
\label{Conclusion}

\subsection{Dataset and methodology}
Our contributions are a crucial step towards a principled evaluation methodology through which estimating the specificity of unsupervised detection in arbitrarily sized datasets is reduced to the simpler problem of binary classification of ND vs. NND pairs; a tractable number of NND pairs can be extracted through hard negative mining strategies. In the simplest implementation, hard negatives can be mined by finding the nearest neighbors in the dataset, using exact or approximate search depending on the size. 

We established the first benchmark for unsupervised NIND detection, an extension of the MFND benchmark comprising  {more than 20,000 pairs} of INDs or NINDs.  We followed a semi-automatic procedure that potentially could locate almost all pairs of NDs in the dataset \citep{connor2015identification}. In our experiments,  {occasionally hard negatives mined may still contain a small percentage (1-2\%) of NDs}: hence, annotation of the MFND benchmark should be regarded as an ongoing process, that will grow as new descriptors will be tested. For comparison, in an initial experiment performed before extending the dataset \citep{connor2015identification}, roughly 8\% of the hard negative mined were either NIND or IND pairs. 
 {Our experimental comparison on state-of-the-art descriptors suggests that, when compared with a real-life dataset representative of a fraud detection application, MFND is a surprisingly realistic benchmark for estimating the specificity. On the contrary, NIND samples in MFND are on average slightly easier to detect than other datasets, albeit the difference is much reduced compared to IND samples.}
The presented methodology builds upon previous results from \cite{connor2016quantifying} on IND detection; we proved that the accuracy of the estimated specificity crucially depends on choosing a proper hard negative mining strategy.  {We provide two additional contributions that strengthen the adoption of this methodology: first, we show analytically that the AUC of the ROC obtained on the hard negative subset is an upper bound of the true AUC. Secondly, we show experimentally that, starting from the experimental ROC, we are able to predict quite accurately the false positive rate per query, which is an indirect proof that the ROC is indeed a good approximation of the true curve. For this experimental comparison, we used the same dataset for hard negative mining and performance evaluation, but in principle, it would be more convenient to perform the hard negative mining on a smaller dataset.  Future work is needed to determine whether the false positive rate can be extrapolated to a larger dataset.}

An alternative, more intuitive, figure of merit would be the average recall and FPs/query as a function of the distance threshold $t$. This curve is less practical to use as it depends on the size of the dataset and, being unbounded, defining summary performance measures such as the Area under the ROC curve is not straightforward. It closely resembles the Free-Response Receiver Operating Characteristics (FROC), an extension of ROC analysis used for many diagnostic tasks where the observer (human or machine) can identify the location of an arbitrary number of potential abnormalities, as opposed to the binary prediction task of determining whether an abnormality is present or not \citep{petrick2013evaluation}. In that context, alternatives to the AUC have been proposed and could be extended to our use case. 
\subsection{Performance comparison}
To the best of our knowledge, this the first attempt to evaluate deep learning descriptors on unsupervised discovery of non-identical near duplicates. 

\citet{connor2016quantifying} argued that global descriptors are sufficient for IND detection. Our experience on the GIST descriptor, which obtained the highest performance in  {the} previous comparison, suggest {s} that CNN-based descriptors offer significant advantages also in this case, and compare favorably in terms of execution time.

 {We have included in our comparison three widely used architecture: SPoC, R-MAC and DeepRetrieval. Note that the DeepRetrieval architecture includes region pooling (like R-MAC), but unlike other descriptors the features are fine-tuned on the Landmarks dataset for the retrieval task using a Siamese network.}
Confirming previous results on instance-level image retrieval benchmarks, nicely summarized by \citet{zheng2017sift}, our experimental results overall favor the choice of fine-tuning the representation for retrieval, as opposed to using off-the-shelf features trained using classification loss \citep{gordo2016deep}. The actual performance gap, however, strongly depends on factors related to both the network architecture, the chosen trade-off between specificity and sensitivity, and the underlying dataset structure.

The Holidays dataset has been extensively used to benchmark instance-level retrieval tasks, and all descriptors analyzed in this paper were also previously tested on this dataset, albeit using a different approach for performance assessment. The performance (mean Average Precision) is reported in previous literature as follows: 75.9 (SPoC), 85.2 (R-MAC) and 86.7 (DeepRetrieval) \citep{gordo2016deep}. For the Holidays near-duplicate detection task, the best results  for the three descriptors are 0.641 (R-MAC), 0.694 (SPoC) and 0.676 (DeepRetrieval), suggesting that SPoC may outperform architectures that are significantly more complex to train and deploy. We should note that none of the descriptors were trained on the Holidays dataset, but the PCA for SPoC and R-MAC was trained on the MFND dataset, which is used as distractors for the near-duplicate detection task. 

First, the task is different, not only because the performance measure is different, but also because in our experimental setting, images from the MFND collection are used as negative samples; this is needed to evaluate specificity, which is difficult to do directly on Holidays due to the small size of the dataset and the absence of distractor images. We found  experimentally that in many cases the increase in sensitivity is counterbalanced by a corresponding increase in the false positive rate. This is especially evident for the R-MAC descriptor, for which the overall performance decreases in all datasets except MFND.  
Secondly, each descriptor has many parameters, and the best combination is dataset dependent. While exploring all possible combinations is a daunting task, our experiments provide some useful insights.  
We found that the backbone depth and architecture were the single most important factor affecting performance {. The} original SPoC paper, and many subsequent comparisons \citep{babenko2015aggregating,gordo2016deep, zheng2017sift}, employed the VGG architecture as backbone, but we found a major boost in performance by using Residual Networks; the DeepRetrieval architecture, on the contrary, uses ResNet101 as backbone \citep{gordo2017end}. In our experiments, the depth of the architecture appears a more relevant factor than the specific feature training, and this an important consideration that should be kept in mind by practitioners.

When compared on the same backbone architecture (Resnet101), the DeepRetrieval outperformed SPoC on CLAIMS and MFND, but not on Holidays. The Holidays dataset contains a lot of outdoors and natural scenes imagery, which may not sufficiently covered by the Landmarks dataset. We expected that SPoC features extracted from networks trained on a scene recognition task, for instance on the Places dataset, or a mixture of Places and ImageNet, could perform better for near-duplicate detection, since many near-duplicates include complex scenes. However, we did not find consistent advantages, especially when using Residual Networks as the backbone architecture. 
 
In a high specificity setting, the difference between pre-trained and fine-tuned networks is further reduced, as visually similar images tend to generate many false positives. Future work will be dedicated to training a specific descriptor for unsupervised near-duplicate detection, incorporating specificity requirements at training time as well as test time.  {In literature, feature weighting schemes have also been proposed \citep{mohedano2018saliency,kalantidis2016cross}; such descriptors could be trained in an unsupervised fashion, or do not require any training at all. The performance of such schemes from the point of view of specificity is another direction worth exploring.}

 {In this work, we have used the same descriptor and distance function for all images, regarding of their content.} Notably, images are not uniformly distributed in the embedded feature space, and the specificity is largely affected by the presence of clusters of images that are very similar from a semantic and visual point of view. This behaviour is observed in both CLAIMS and MFND datasets, despite their different origin. Exploiting this underlying structure to improve the performance of ND discovery is an important avenue for future research. 
 
\section{Conclusions}
Unsupervised discovery of near-duplicate detection is an important problem in digital forensics and fraud detection. As the number of false alarms grows quadratically with the size of the input dataset, practical applications require a very high specificity, or conversely low false positive rate, often in the range of $10^{-7}$--$10^{-10}$. Hard negative mining can be used to select a subset of the dataset, on which ROC analysis can be used to evaluate the performance.

We have evaluated a selection of descriptors based on Convolutional Neural Networks following the proposed methodology. While the task of NIND detection is conceptually similar to instance-level image retrieval, we experimentally found that the same descriptors may be ranked differently, as the Area under the ROC curve depends more strongly on specificity than the mean Average Precision. This strengthens the need for a dedicated benchmark, targeting applications where unsupervised search is required. 
Our findings in general favor the choice of fine-tuning deep convolutional networks, as opposed to using off-the-shelf features, but differences at high specificity settings strongly depend on the specific dataset and are often small. On the MFND collection, promising performance is obtained by the DeepRet descriptor, retrieving 96\% of the true positives at a FP rate of $1.43 \times 10  ^{-6}$. However, further improvement in specificity would benefit many applications, especially in the forensics domain.

\appendix
\section{Hard negative mining provides an upper bound for the AUC}
\label{sec:app1}

In this section, proof that the AUC calculated using either hard negative mining strategies is an upper bound for the true AUC is provided. 

\newtheorem{prop}{Proposition}
\begin{prop}
When using hard negative mining strategy $hn2$, the resulting $AUC_{hn2}$ is an upper bound for the true $AUC$.
\end{prop}

\begin{proof}
Hard negative mining strategy $hn2$ ensures that the selected $n_l$ pairs are the most difficult pairs within the set $\mathbbm{N^-}$; it follows that:
\begin{equation}
    f(n_j) \leq f(n_l) \quad  \forall n_j \in \mathbbm{N^- - H^-} \quad  \forall n_l \in \mathbbm{H^-}
\end{equation}
and consequently:
\begin{equation}
    \1{p_i > n_j} \leq \1{p_i > n_l}  \quad  \forall n_j \in \mathbbm{N^- - H^-} \quad  \forall n_l \in \mathbbm{H^-}
    \label{eq:inequality_hn2}
\end{equation}

The AUC can be decomposed in two terms
\begin{equation}
    AUC = \frac{1}{N^+N^-}\bigg[ \sum_{i=1}^{N^+} \sum_{l=1}^{H^-}\1{p_i > n_l} + \sum_{i=1}^{N^+} \sum_{j=1}^{N^- - H^-}\1{p_i > n_j}  \bigg] 
    \label{eq:auchn2}
\end{equation}
where the first term is known and is proportional to $AUC_{hn2}$ from Eq. \ref{eq:auch}, and the second term is the contribution of the negative samples that are not observed. However, we can substitute the second term by replicating the hard negative samples $\ceil[\big]{\frac{N^--H^-}{H^-}}$ times, and combining with Eq. \ref{eq:inequality_hn2} we conclude that:
\begin{equation*}
\begin{aligned}
    AUC \leq \frac{1}{N^+N^-}\bigg[ \sum_{i=1}^{N^+} \sum_{l=1}^{H^-}\1{p_i > n_l} + \sum_{k=1}^{\frac{N^--H^-}{H^-}}\sum_{i=1}^{N^+} \sum_{l=1}^{H^-}\1{p_i > n_l}  \bigg] = \\
    = \frac{1}{N^+N^-}\bigg[ N^+H^- AUC_{hn2}  + \big( \frac{N^--H^-}{H^-} N^+H^- \big) AUC_{hn2} \bigg] = AUC_{hn2}
\end{aligned}
\end{equation*}
\end{proof}

\begin{prop}
When using hard negative mining strategy $hn1$, the resulting $AUC_{hn1}$ is an upper bound for the true $AUC$.
\end{prop}

\begin{proof}
Again, let us decompose the AUC as the sum of two terms, where the first term is known and is proportional to $AUC_{hn1}$, and the second term is the contribution of the negative samples that are not observed, as detailed in Eq. \ref{eq:auchn2}.

Each sample $n_j$ consists of a pair of images $(x_k,y_m)$, where $x_k \in X$ and $y_m \in Y$; in other terms, $N^- =\{(x_k, y_m), \quad k=1,...K, \quad  m=1,...,M  \}$. Then according to the definition of $hn1$,
\begin{equation}
    H^- = \{(x_k, y_{m^*}) \mid  m^* = \argmax_{m}f(x_k,y_m)\}
\end{equation}
The sum over $l=1,...,H^-$ and $j=1,...,N^-$ in Eq. \ref{eq:auchn2} can be decomposed in terms of $k=1,...,K$ and $m=1,...,M$ as follows:
\begin{equation}
AUC = \frac{1}{N^+N^-}\bigg[ \sum_{i=1}^{N^+}\sum_{k=1}^{K}\1{p_i > (x_k, y_{m^*})} + \sum_{i=1}^{N^+} \sum_{k=1}^{K}\sum_{m=1, m \neq m^*}^{M}\1{p_i > (x_k, y_m)} \bigg]
\label{eq:auchn1}
\end{equation}
where $N^- = K \times M$, and according to the definition of \textit{hn1} there are exactly $K$ hard negative pairs.

By definition, $ f(x_k,y_m) \leq f(x_k, y_{m^*})$ and thus 
\begin{equation}
    \1{p_i > (x_k,y_m)} \leq \1{p_i > (x_k, y_{m^*})}  \quad  \forall k=1,...,K \quad  \forall m \neq m^*
    \label{eq:inequality2}
\end{equation}
Combining Eqs. \ref{eq:auchn1} and \ref{eq:inequality2}, we conclude that:
\begin{equation*}
\begin{aligned}
  AUC \leq \frac{1}{N^+N^-}\bigg[ N^+H^- AUC_{hn1}  + \sum_{i=1}^{N^+} \sum_{k=1}^{K}\sum_{m=1, m \neq m^*}^{M}\1{p_i > (x_k,y_{m^*})} \bigg] =  \\
  \frac{1}{N^+N^-}\bigg[ N^+K AUC_{hn1}  + (M-1)N^+KAUC_{hn1} \bigg] = AUC_{hn1}
\end{aligned}
\end{equation*}

\end{proof}

\section*{Acknowledgment}
The authors gratefully acknowledge the financial support of Reale Mutua Assicurazioni. We are grateful to Lucia Sabatino, Lucia Romano and Giulia Gemesio for help in data cleaning and annotation. 

\section*{References}

\bibliography{benchmarking}

\begin{thebibliography}{49}
\expandafter\ifx\csname natexlab\endcsname\relax\def\natexlab#1{#1}\fi
\providecommand{\url}[1]{\texttt{#1}}
\providecommand{\href}[2]{#2}
\providecommand{\path}[1]{#1}
\providecommand{\DOIprefix}{doi:}
\providecommand{\ArXivprefix}{arXiv:}
\providecommand{\URLprefix}{URL: }
\providecommand{\Pubmedprefix}{pmid:}
\providecommand{\doi}[1]{\href{http://dx.doi.org/#1}{\path{#1}}}
\providecommand{\Pubmed}[1]{\href{pmid:#1}{\path{#1}}}
\providecommand{\bibinfo}[2]{#2}
\ifx\xfnm\relax \def\xfnm[#1]{\unskip,\space#1}\fi
\bibitem[{Amerini et~al.(2017)Amerini, Uricchio \&
  Caldelli}]{amerini2017tracing}
\bibinfo{author}{Amerini, I.}, \bibinfo{author}{Uricchio, T.}, \&
  \bibinfo{author}{Caldelli, R.} (\bibinfo{year}{2017}).
\newblock \bibinfo{title}{Tracing images back to their social network of
  origin: A cnn-based approach}.
\newblock In {\it \bibinfo{booktitle}{Information Forensics and Security
  (WIFS), 2017 IEEE Workshop on}\/} (pp. \bibinfo{pages}{1--6}).
\newblock 
\bibitem[{Babenko \& Lempitsky(2015)}]{babenko2015aggregating}
\bibinfo{author}{Babenko, A.}, \& \bibinfo{author}{Lempitsky, V.}
  (\bibinfo{year}{2015}).
\newblock \bibinfo{title}{Aggregating local deep features for image retrieval}.
\newblock In {\it \bibinfo{booktitle}{Proceedings of the IEEE international
  conference on computer vision}\/} 
\bibitem[{Babenko et~al.(2014)Babenko, Slesarev, Chigorin \&
  Lempitsky}]{babenko2014neural}
\bibinfo{author}{Babenko, A.}, \bibinfo{author}{Slesarev, A.},
  \bibinfo{author}{Chigorin, A.}, \& \bibinfo{author}{Lempitsky, V.}
  (\bibinfo{year}{2014}).
\newblock \bibinfo{title}{Neural codes for image retrieval}.
\newblock In {\it \bibinfo{booktitle}{European conference on computer
  vision}\/} (pp. \bibinfo{pages}{584--599}).
\newblock 
\bibitem[{Balntas et~al.(2016)Balntas, Riba, Ponsa \&
  Mikolajczyk}]{balntas2016learning}
\bibinfo{author}{Balntas, V.}, \bibinfo{author}{Riba, E.},
  \bibinfo{author}{Ponsa, D.}, \& \bibinfo{author}{Mikolajczyk, K.}
  (\bibinfo{year}{2016}).
\newblock \bibinfo{title}{Learning local feature descriptors with triplets and
  shallow convolutional neural networks.}
\newblock In {\it \bibinfo{booktitle}{BMVC}\/} (p.~\bibinfo{pages}{3}).
\newblock 
\bibitem[{Battiato et~al.(2014)Battiato, Farinella, Puglisi \&
  Rav{\`\i}}]{battiato2014aligning}
\bibinfo{author}{Battiato, S.}, \bibinfo{author}{Farinella, G.~M.},
  \bibinfo{author}{Puglisi, G.}, \& \bibinfo{author}{Rav{\`\i}, D.}
  (\bibinfo{year}{2014}).
\newblock \bibinfo{title}{Aligning codebooks for near duplicate image
  detection}.
\newblock {\it \bibinfo{journal}{Multimedia Tools and Applications}\/},  
  = Article
\bibitem[{Bay et~al.(2008)Bay, Ess, Tuytelaars \& Van~Gool}]{bay2008speeded}
\bibinfo{author}{Bay, H.}, \bibinfo{author}{Ess, A.},
  \bibinfo{author}{Tuytelaars, T.}, \& \bibinfo{author}{Van~Gool, L.}
  (\bibinfo{year}{2008}).
\newblock \bibinfo{title}{Speeded-up robust features (surf)}.
\newblock {\it \bibinfo{journal}{Computer vision and image understanding}\/},
\bibitem[{Bradley(1997)}]{bradley1997use}
\bibinfo{author}{Bradley, A.~P.} (\bibinfo{year}{1997}).
\newblock \bibinfo{title}{The use of the area under the roc curve in the
  evaluation of machine learning algorithms}.
\newblock {\it \bibinfo{journal}{Pattern recognition}\/},  
\bibitem[{Chen et~al.(2017)Chen, Li, Zhang, Hsu \& Wang}]{chen2017real}
\bibinfo{author}{Chen, M.}, \bibinfo{author}{Li, Y.}, \bibinfo{author}{Zhang,
  Z.}, \bibinfo{author}{Hsu, C.-H.}, \& \bibinfo{author}{Wang, S.}
  (\bibinfo{year}{2017}).
\newblock \bibinfo{title}{Real-time, large-scale duplicate image detection
  method based on multi-feature fusion}.
\newblock {\it \bibinfo{journal}{Journal of Real-Time Image Processing}\/},
\bibitem[{Chennamma et~al.(2009)Chennamma, Rangarajan \&
  Rao}]{chennamma2009robust}
\bibinfo{author}{Chennamma, H.}, \bibinfo{author}{Rangarajan, L.}, \&
  \bibinfo{author}{Rao, M.} (\bibinfo{year}{2009}).
\newblock \bibinfo{title}{Robust near duplicate image matching for digital
  image forensics}.
\newblock {\it \bibinfo{journal}{International Journal of Digital Crime and
  Forensics (IJDCF)}\/},  
\bibitem[{Chiu et~al.(2012)Chiu, Li \& Hsieh}]{chiu2012video}
\bibinfo{author}{Chiu, C.-Y.}, \bibinfo{author}{Li, S.-Y.}, \&
  \bibinfo{author}{Hsieh, C.-Y.} (\bibinfo{year}{2012}).
\newblock \bibinfo{title}{Video query reformulation for near-duplicate
  detection}.
\newblock {\it \bibinfo{journal}{IEEE Transactions on Information Forensics and
  Security}\/},  
\bibitem[{Chum et~al.(2008)Chum, Philbin, Zisserman et~al.}]{chum2008near}
\bibinfo{author}{Chum, O.}, \bibinfo{author}{Philbin, J.},
  \bibinfo{author}{Zisserman, A.} et~al. (\bibinfo{year}{2008}).
\newblock \bibinfo{title}{Near duplicate image detection: min-hash and tf-idf
  weighting.}
\newblock In {\it \bibinfo{booktitle}{BMVC}\/} (pp. \bibinfo{pages}{812--815}).
\newblock 
\bibitem[{Cicconet et~al.(2018)Cicconet, Elliott, Richmond, Wainstock \&
  Walsh}]{cicconet2018image}
\bibinfo{author}{Cicconet, M.}, \bibinfo{author}{Elliott, H.},
  \bibinfo{author}{Richmond, D.~L.}, \bibinfo{author}{Wainstock, D.}, \&
  \bibinfo{author}{Walsh, M.} (\bibinfo{year}{2018}).
\newblock \bibinfo{title}{Image forensics: Detecting duplication of scientific
  images with manipulation-invariant image similarity}.
\newblock 
\bibitem[{Connor et~al.(2015)Connor, Cardillo, MacKenzie-Leigh \&
  Moss}]{connor2015identification}
\bibinfo{author}{Connor, R.}, \bibinfo{author}{Cardillo, F.},
  \bibinfo{author}{MacKenzie-Leigh, S.}, \& \bibinfo{author}{Moss, R.}
  (\bibinfo{year}{2015}).
\newblock \bibinfo{title}{Identification of mir-flickr near-duplicate images}.
\newblock 
\bibitem[{Connor \& Cardillo(2016)}]{connor2016quantifying}
\bibinfo{author}{Connor, R.}, \& \bibinfo{author}{Cardillo, F.~A.}
  (\bibinfo{year}{2016}).
\newblock \bibinfo{title}{Quantifying the specificity of near-duplicate image
  classification functions}.
\newblock 
\bibitem[{Deng et~al.(2009)Deng, Dong, Socher, Li, Li \&
  Fei-Fei}]{deng2009imagenet}
\bibinfo{author}{Deng, J.}, \bibinfo{author}{Dong, W.},
  \bibinfo{author}{Socher, R.}, \bibinfo{author}{Li, L.-J.},
  \bibinfo{author}{Li, K.}, \& \bibinfo{author}{Fei-Fei, L.}
  (\bibinfo{year}{2009}).
\newblock \bibinfo{title}{Imagenet: A large-scale hierarchical image database}.
\newblock In {\it \bibinfo{booktitle}{Computer Vision and Pattern Recognition,
  2009. CVPR 2009. IEEE Conference on}\/} (pp. \bibinfo{pages}{248--255}).
\newblock 
\bibitem[{Foo \& Sinha(2007)}]{foo2007pruning}
\bibinfo{author}{Foo, J.~J.}, \& \bibinfo{author}{Sinha, R.}
  (\bibinfo{year}{2007}).
\newblock \bibinfo{title}{Pruning sift for scalable near-duplicate image
  matching}.
\newblock In {\it \bibinfo{booktitle}{Proceedings of the eighteenth conference
  on Australasian database-Volume 63}\/} (pp. \bibinfo{pages}{63--71}).
\newblock 
\bibitem[{Gon{\c{c}}alves et~al.(2018)Gon{\c{c}}alves, Guilherme \&
  Pedronette}]{gonccalves2018semantic}
\bibinfo{author}{Gon{\c{c}}alves, F. M.~F.}, \bibinfo{author}{Guilherme,
  I.~R.}, \& \bibinfo{author}{Pedronette, D. C.~G.} (\bibinfo{year}{2018}).
\newblock \bibinfo{title}{Semantic guided interactive image retrieval for plant
  identification}.
\newblock {\it \bibinfo{journal}{Expert Systems with Applications}\/},  
  Inproceedings
\bibitem[{Gordo et~al.(2016)Gordo, Almaz{\'a}n, Revaud \&
  Larlus}]{gordo2016deep}
\bibinfo{author}{Gordo, A.}, \bibinfo{author}{Almaz{\'a}n, J.},
  \bibinfo{author}{Revaud, J.}, \& \bibinfo{author}{Larlus, D.}
  (\bibinfo{year}{2016}).
\newblock \bibinfo{title}{Deep image retrieval: Learning global representations
  for image search}.
\newblock In {\it \bibinfo{booktitle}{European Conference on Computer
  Vision}\/} (pp. \bibinfo{pages}{241--257}).
\newblock 
\bibitem[{Gordo et~al.(2017)Gordo, Almazan, Revaud \& Larlus}]{gordo2017end}
\bibinfo{author}{Gordo, A.}, \bibinfo{author}{Almazan, J.},
  \bibinfo{author}{Revaud, J.}, \& \bibinfo{author}{Larlus, D.}
  (\bibinfo{year}{2017}).
\newblock \bibinfo{title}{End-to-end learning of deep visual representations
  for image retrieval}.
\newblock {\it \bibinfo{journal}{International Journal of Computer Vision}\/},
\bibitem[{Hanley \& McNeil(1982)}]{hanley1982meaning}
\bibinfo{author}{Hanley, J.~A.}, \& \bibinfo{author}{McNeil, B.~J.}
  (\bibinfo{year}{1982}).
\newblock \bibinfo{title}{The meaning and use of the area under a receiver
  operating characteristic (roc) curve.}
\newblock {\it \bibinfo{journal}{Radiology}\/},  
\bibitem[{He et~al.(2016)He, Zhang, Ren \& Sun}]{he2016deep}
\bibinfo{author}{He, K.}, \bibinfo{author}{Zhang, X.}, \bibinfo{author}{Ren,
  S.}, \& \bibinfo{author}{Sun, J.} (\bibinfo{year}{2016}).
\newblock \bibinfo{title}{Deep residual learning for image recognition}.
\newblock In {\it \bibinfo{booktitle}{Proceedings of the IEEE conference on
  computer vision and pattern recognition}\/} 
\bibitem[{Hirano et~al.(2006)Hirano, Garcia, Sukthankar \&
  Hoogs}]{hirano2006industry}
\bibinfo{author}{Hirano, Y.}, \bibinfo{author}{Garcia, C.},
  \bibinfo{author}{Sukthankar, R.}, \& \bibinfo{author}{Hoogs, A.}
  (\bibinfo{year}{2006}).
\newblock \bibinfo{title}{Industry and object recognition: Applications,
  applied research and challenges}.
\newblock In {\it \bibinfo{booktitle}{Toward Category-Level Object
  Recognition}\/} (pp. \bibinfo{pages}{49--64}).
\newblock 
\bibitem[{Hu et~al.(2009)Hu, Cheng, Chia, Xie, Rajan \& Tan}]{hu2009coherent}
\bibinfo{author}{Hu, Y.}, \bibinfo{author}{Cheng, X.}, \bibinfo{author}{Chia,
  L.-T.}, \bibinfo{author}{Xie, X.}, \bibinfo{author}{Rajan, D.}, \&
  \bibinfo{author}{Tan, A.-H.} (\bibinfo{year}{2009}).
\newblock \bibinfo{title}{Coherent phrase model for efficient image
  near-duplicate retrieval.}
\newblock {\it \bibinfo{journal}{IEEE Trans. Multimedia}\/},  
  Inproceedings
\bibitem[{Huiskes \& Lew(2008)}]{huiskes2008mir}
\bibinfo{author}{Huiskes, M.~J.}, \& \bibinfo{author}{Lew, M.~S.}
  (\bibinfo{year}{2008}).
\newblock \bibinfo{title}{The mir flickr retrieval evaluation}.
\newblock In {\it \bibinfo{booktitle}{Proceedings of the 1st ACM international
  conference on Multimedia information retrieval}\/} (pp.
  \bibinfo{pages}{39--43}).
\newblock 
\bibitem[{Jegou et~al.(2008)Jegou, Douze \& Schmid}]{jegou2008hamming}
\bibinfo{author}{Jegou, H.}, \bibinfo{author}{Douze, M.}, \&
  \bibinfo{author}{Schmid, C.} (\bibinfo{year}{2008}).
\newblock \bibinfo{title}{Hamming embedding and weak geometric consistency for
  large scale image search}.
\newblock In {\it \bibinfo{booktitle}{European conference on computer
  vision}\/} (pp. \bibinfo{pages}{304--317}).
\newblock 
\bibitem[{Jinda-Apiraksa et~al.(2013)Jinda-Apiraksa, Vonikakis \&
  Winkler}]{jinda2013california}
\bibinfo{author}{Jinda-Apiraksa, A.}, \bibinfo{author}{Vonikakis, V.}, \&
  \bibinfo{author}{Winkler, S.} (\bibinfo{year}{2013}).
\newblock \bibinfo{title}{California-nd: An annotated dataset for
  near-duplicate detection in personal photo collections}.
\newblock In {\it \bibinfo{booktitle}{Quality of Multimedia Experience (QoMEX),
  2013 Fifth International Workshop on}\/} (pp. \bibinfo{pages}{142--147}).
\newblock 
\bibitem[{Johnson et~al.(2017)Johnson, Douze \& J{\'e}gou}]{johnson2017billion}
\bibinfo{author}{Johnson, J.}, \bibinfo{author}{Douze, M.}, \&
  \bibinfo{author}{J{\'e}gou, H.} (\bibinfo{year}{2017}).
\newblock \bibinfo{title}{Billion-scale similarity search with gpus}.
\newblock 
\bibitem[{Kalantidis et~al.(2016)Kalantidis, Mellina \&
  Osindero}]{kalantidis2016cross}
\bibinfo{author}{Kalantidis, Y.}, \bibinfo{author}{Mellina, C.}, \&
  \bibinfo{author}{Osindero, S.} (\bibinfo{year}{2016}).
\newblock \bibinfo{title}{Cross-dimensional weighting for aggregated deep
  convolutional features}.
\newblock In {\it \bibinfo{booktitle}{European conference on computer
  vision}\/} (pp. \bibinfo{pages}{685--701}).
\newblock 
\bibitem[{Ke et~al.(2004)Ke, Sukthankar \& Huston}]{ke2004efficient}
\bibinfo{author}{Ke, Y.}, \bibinfo{author}{Sukthankar, R.}, \&
  \bibinfo{author}{Huston, L.} (\bibinfo{year}{2004}).
\newblock \bibinfo{title}{An efficient parts-based near-duplicate and sub-image
  retrieval system}.
\newblock In {\it \bibinfo{booktitle}{Proceedings of the 12th annual ACM
  international conference on Multimedia}\/} (pp. \bibinfo{pages}{869--876}).
\newblock 
\bibitem[{Kim et~al.(2015)Kim, Wang, Zhang \& Choi}]{kim2015near}
\bibinfo{author}{Kim, S.}, \bibinfo{author}{Wang, X.-J.},
  \bibinfo{author}{Zhang, L.}, \& \bibinfo{author}{Choi, S.}
  (\bibinfo{year}{2015}).
\newblock \bibinfo{title}{Near duplicate image discovery on one billion
  images}.
\newblock In {\it \bibinfo{booktitle}{Applications of Computer Vision (WACV),
  2015 IEEE Winter Conference on}\/} (pp. \bibinfo{pages}{943--950}).
\newblock 
\bibitem[{Li et~al.(2015)Li, Qian, Li, Zhao, Wang \& Tang}]{li2015mining}
\bibinfo{author}{Li, J.}, \bibinfo{author}{Qian, X.}, \bibinfo{author}{Li, Q.},
  \bibinfo{author}{Zhao, Y.}, \bibinfo{author}{Wang, L.}, \&
  \bibinfo{author}{Tang, Y.~Y.} (\bibinfo{year}{2015}).
\newblock \bibinfo{title}{Mining near duplicate image groups}.
\newblock {\it \bibinfo{journal}{Multimedia Tools and Applications}\/},  
  = Article
\bibitem[{Li et~al.(2018)Li, Shen \& Dong}]{li2018anti}
\bibinfo{author}{Li, P.}, \bibinfo{author}{Shen, B.}, \& \bibinfo{author}{Dong,
  W.} (\bibinfo{year}{2018}).
\newblock \bibinfo{title}{An anti-fraud system for car insurance claim based on
  visual evidence}.
\newblock 
\bibitem[{Liu et~al.(2015)Liu, Lu \& Suen}]{liu2015variable}
\bibinfo{author}{Liu, L.}, \bibinfo{author}{Lu, Y.}, \& \bibinfo{author}{Suen,
  C.~Y.} (\bibinfo{year}{2015}).
\newblock \bibinfo{title}{Variable-length signature for near-duplicate image
  matching}.
\newblock {\it \bibinfo{journal}{IEEE Transactions on Image Processing}\/},
\bibitem[{Mohedano et~al.(2018)Mohedano, McGuinness, Gir{\'o}-i Nieto \&
  O'Connor}]{mohedano2018saliency}
\bibinfo{author}{Mohedano, E.}, \bibinfo{author}{McGuinness, K.},
  \bibinfo{author}{Gir{\'o}-i Nieto, X.}, \& \bibinfo{author}{O'Connor, N.~E.}
  (\bibinfo{year}{2018}).
\newblock \bibinfo{title}{Saliency weighted convolutional features for instance
  search}.
\newblock In {\it \bibinfo{booktitle}{2018 International Conference on
  Content-Based Multimedia Indexing (CBMI)}\/} (pp. \bibinfo{pages}{1--6}).
\newblock 
\bibitem[{Oliva \& Torralba(2001)}]{oliva2001modeling}
\bibinfo{author}{Oliva, A.}, \& \bibinfo{author}{Torralba, A.}
  (\bibinfo{year}{2001}).
\newblock \bibinfo{title}{Modeling the shape of the scene: A holistic
  representation of the spatial envelope}.
\newblock {\it \bibinfo{journal}{International journal of computer vision}\/},
\bibitem[{de~Oliveira et~al.(2016)de~Oliveira, Ferrara, De~Rosa, Piva, Barni,
  Goldenstein, Dias \& Rocha}]{de2016multiple}
\bibinfo{author}{de~Oliveira, A.~A.}, \bibinfo{author}{Ferrara, P.},
  \bibinfo{author}{De~Rosa, A.}, \bibinfo{author}{Piva, A.},
  \bibinfo{author}{Barni, M.}, \bibinfo{author}{Goldenstein, S.},
  \bibinfo{author}{Dias, Z.}, \& \bibinfo{author}{Rocha, A.}
  (\bibinfo{year}{2016}).
\newblock \bibinfo{title}{Multiple parenting phylogeny relationships in digital
  images}.
\newblock {\it \bibinfo{journal}{IEEE Transactions on Information Forensics and
  Security}\/},  
\bibitem[{Petrick et~al.(2013)Petrick, Sahiner, Armato, Bert, Correale,
  Delsanto, Freedman, Fryd, Gur, Hadjiiski et~al.}]{petrick2013evaluation}
\bibinfo{author}{Petrick, N.}, \bibinfo{author}{Sahiner, B.},
  \bibinfo{author}{Armato, S.~G.}, \bibinfo{author}{Bert, A.},
  \bibinfo{author}{Correale, L.}, \bibinfo{author}{Delsanto, S.},
  \bibinfo{author}{Freedman, M.~T.}, \bibinfo{author}{Fryd, D.},
  \bibinfo{author}{Gur, D.}, \bibinfo{author}{Hadjiiski, L.} et~al.
  (\bibinfo{year}{2013}).
\newblock \bibinfo{title}{Evaluation of computer-aided detection and diagnosis
  systems}.
\newblock {\it \bibinfo{journal}{Medical physics}\/},  
\bibitem[{Simonyan \& Zisserman(2014)}]{simonyan2014very}
\bibinfo{author}{Simonyan, K.}, \& \bibinfo{author}{Zisserman, A.}
  (\bibinfo{year}{2014}).
\newblock \bibinfo{title}{Very deep convolutional networks for large-scale
  image recognition}.
\newblock 
\bibitem[{Tolias et~al.(2015)Tolias, Sicre \& J{\'e}gou}]{tolias2015particular}
\bibinfo{author}{Tolias, G.}, \bibinfo{author}{Sicre, R.}, \&
  \bibinfo{author}{J{\'e}gou, H.} (\bibinfo{year}{2015}).
\newblock \bibinfo{title}{Particular object retrieval with integral max-pooling
  of cnn activations}.
\newblock 
\bibitem[{Wan et~al.(2014)Wan, Wang, Hoi, Wu, Zhu, Zhang \& Li}]{wan2014deep}
\bibinfo{author}{Wan, J.}, \bibinfo{author}{Wang, D.}, \bibinfo{author}{Hoi, S.
  C.~H.}, \bibinfo{author}{Wu, P.}, \bibinfo{author}{Zhu, J.},
  \bibinfo{author}{Zhang, Y.}, \& \bibinfo{author}{Li, J.}
  (\bibinfo{year}{2014}).
\newblock \bibinfo{title}{Deep learning for content-based image retrieval: A
  comprehensive study}.
\newblock In {\it \bibinfo{booktitle}{Proceedings of the 22nd ACM international
  conference on Multimedia}\/} (pp. \bibinfo{pages}{157--166}).
\newblock 
\bibitem[{Wang et~al.(2014)Wang, Song, Leung, Rosenberg, Wang, Philbin, Chen \&
  Wu}]{wang2014learning}
\bibinfo{author}{Wang, J.}, \bibinfo{author}{Song, Y.}, \bibinfo{author}{Leung,
  T.}, \bibinfo{author}{Rosenberg, C.}, \bibinfo{author}{Wang, J.},
  \bibinfo{author}{Philbin, J.}, \bibinfo{author}{Chen, B.}, \&
  \bibinfo{author}{Wu, Y.} (\bibinfo{year}{2014}).
\newblock \bibinfo{title}{Learning fine-grained image similarity with deep
  ranking}.
\newblock In {\it \bibinfo{booktitle}{Proceedings of the IEEE Conference on
  Computer Vision and Pattern Recognition}\/} 
\bibitem[{Xie et~al.(2014)Xie, Tian, Zhou \& Zhang}]{xie2014fast}
\bibinfo{author}{Xie, L.}, \bibinfo{author}{Tian, Q.}, \bibinfo{author}{Zhou,
  W.}, \& \bibinfo{author}{Zhang, B.} (\bibinfo{year}{2014}).
\newblock \bibinfo{title}{Fast and accurate near-duplicate image search with
  affinity propagation on the imageweb}.
\newblock {\it \bibinfo{journal}{Computer Vision and Image Understanding}\/},
\bibitem[{Xu et~al.(2010)Xu, Cham, Yan, Duan \& Chang}]{xu2010near}
\bibinfo{author}{Xu, D.}, \bibinfo{author}{Cham, T.~J.}, \bibinfo{author}{Yan,
  S.}, \bibinfo{author}{Duan, L.}, \& \bibinfo{author}{Chang, S.-F.}
  (\bibinfo{year}{2010}).
\newblock \bibinfo{title}{Near duplicate identification with spatially aligned
  pyramid matching}.
\newblock {\it \bibinfo{journal}{IEEE Transactions on Circuits and Systems for
  Video Technology}\/},  
\bibitem[{Zagoruyko \& Komodakis(2015)}]{zagoruyko2015learning}
\bibinfo{author}{Zagoruyko, S.}, \& \bibinfo{author}{Komodakis, N.}
  (\bibinfo{year}{2015}).
\newblock \bibinfo{title}{Learning to compare image patches via convolutional
  neural networks}.
\newblock In {\it \bibinfo{booktitle}{Proceedings of the IEEE Conference on
  Computer Vision and Pattern Recognition}\/} 
\bibitem[{Zheng et~al.(2017)Zheng, Yang \& Tian}]{zheng2017sift}
\bibinfo{author}{Zheng, L.}, \bibinfo{author}{Yang, Y.}, \&
  \bibinfo{author}{Tian, Q.} (\bibinfo{year}{2017}).
\newblock \bibinfo{title}{Sift meets cnn: A decade survey of instance
  retrieval}.
\newblock 
\bibitem[{Zhou et~al.(2017{\natexlab{a}})Zhou, Lapedriza, Khosla, Oliva \&
  Torralba}]{zhou2017places}
\bibinfo{author}{Zhou, B.}, \bibinfo{author}{Lapedriza, A.},
  \bibinfo{author}{Khosla, A.}, \bibinfo{author}{Oliva, A.}, \&
  \bibinfo{author}{Torralba, A.} (\bibinfo{year}{2017}{\natexlab{a}}).
\newblock \bibinfo{title}{Places: A 10 million image database for scene
  recognition}.
\newblock 
\bibitem[{Zhou et~al.(2014)Zhou, Lapedriza, Xiao, Torralba \&
  Oliva}]{zhou2014learning}
\bibinfo{author}{Zhou, B.}, \bibinfo{author}{Lapedriza, A.},
  \bibinfo{author}{Xiao, J.}, \bibinfo{author}{Torralba, A.}, \&
  \bibinfo{author}{Oliva, A.} (\bibinfo{year}{2014}).
\newblock \bibinfo{title}{Learning deep features for scene recognition using
  places database}.
\newblock In {\it \bibinfo{booktitle}{Advances in neural information processing
  systems}\/} 
\bibitem[{Zhou et~al.(2017{\natexlab{b}})Zhou, Li \& Tian}]{zhou2017recent}
\bibinfo{author}{Zhou, W.}, \bibinfo{author}{Li, H.}, \& \bibinfo{author}{Tian,
  Q.} (\bibinfo{year}{2017}{\natexlab{b}}).
\newblock \bibinfo{title}{Recent advance in content-based image retrieval: A
  literature survey}.
\newblock 
\bibitem[{Zhou et~al.(2017{\natexlab{c}})Zhou, Wang, Wu, Yang \&
  Sun}]{zhou2017effective}
\bibinfo{author}{Zhou, Z.}, \bibinfo{author}{Wang, Y.}, \bibinfo{author}{Wu,
  Q.~J.}, \bibinfo{author}{Yang, C.-N.}, \& \bibinfo{author}{Sun, X.}
  (\bibinfo{year}{2017}{\natexlab{c}}).
\newblock \bibinfo{title}{Effective and efficient global context verification
  for image copy detection}.
\newblock {\it \bibinfo{journal}{IEEE Transactions on Information Forensics and
  Security}\/},
\end{thebibliography}

\end{document}